\documentclass[11pt,letterpaper]{article}

\usepackage[margin=1in]{geometry}
\usepackage[T1]{fontenc}
\usepackage[utf8]{inputenc}
\usepackage{lmodern}
\usepackage{amsmath,amssymb,amsfonts,amsthm,mathtools}
\usepackage{bm}
\usepackage{natbib}
\usepackage{graphicx}
\usepackage{float}
\usepackage{xcolor}
\usepackage{multirow,array,booktabs}
\usepackage{enumitem}
\usepackage{subcaption}
\usepackage{algorithm}
\usepackage{algorithmic}
\usepackage{tikz}
\usepackage{url}
\usepackage{indentfirst}

\DeclareMathOperator*{\argmin}{arg\,min}

\newcommand{\tran}{^{\top\kern-\scriptspace}}
\newtheorem{proposition}{Proposition}

\numberwithin{equation}{section}
\title{Cross-Dimension and Cross-Modality Relational Tensor Representation Learning for Multimodal Health Data Modeling}
\author{
Jiawei Guo \qquad
Meng Zhao\thanks{Corresponding author:
\texttt{m.zhao@lehigh.edu}}\\[0.4em]
\normalsize Department of Industrial and
Systems Engineering, Lehigh University\\
\normalsize \texttt{jig526@lehigh.edu}
\quad \texttt{m.zhao@lehigh.edu}
}
\date{September 15, 2026}

\begin{document}
\maketitle
\begin{abstract}
Multimodal health data integrate information from heterogeneous clinical sources, creating substantial opportunities for data-driven decision support. Realizing these opportunities, however, requires effectively modeling complex dependencies that arise both within individual modalities and across modalities. This paper proposes \textbf{R}elational \textbf{T}ensor \textbf{R}epresentation (RTR), a unified framework that jointly learns dimension-specific tensor representations and cross-dimensional and cross-modality relationships. RTR introduces learnable relational bridge matrices that maintain relational information while preserving the structural identities of tensor dimensions and modalities. A structured fusion mechanism further integrates factor-level and relationship-level information into compact representations for downstream prediction. Experiments on simulated data and two multimodal neurophysiological datasets, covering emotion recognition and brain activity classification, demonstrate that RTR consistently outperforms representative tensor learning, multimodal fusion, and Transformer-based baselines. These results highlight the value of explicitly modeling cross-dimensional and cross-modality relationships for learning informative representations from heterogeneous multimodal tensor data. The source code is planned for public release to support reproducibility and future research.
\end{abstract}
\vspace{1em}

\section{Introduction}\label{sec:Intro}
Modern healthcare systems routinely collect high-dimensional and heterogeneous patient information from multiple sources, including physiological signals, medical imaging, wearable devices, laboratory measurements, and electronic health records. Many of these data sources exhibit multidimensional structures that can be naturally represented as tensors \citep{gaw2022multimodal}. For example, physiological signals may vary across channels, time, and frequency, while medical imaging data contain spatial and temporal organization. Integrating these complementary health data provide opportunities to improve clinical decision making and personalize healthcare delivery \citep{zhang2022m3care}. However, learning informative representations from these heterogeneous data remains challenging. Clinically relevant information can be encoded not only in individual measurements but also in dependencies among data dimensions and across modalities \citep{cremonesi2023need}. These dependencies become particularly complex when modalities differ in their dimensional semantics, sampling resolutions, and measurement mechanisms. For example, electroencephalography (EEG) and functional magnetic resonance imaging (fMRI) characterize brain activity through fundamentally different measurement processes and exhibit distinct tensor structures \citep{zhao2023robust}. An effective representation must therefore preserve the structural information associated with individual tensor dimensions while also capturing how information is related both within and across heterogeneous modalities.

Existing representation learning paradigms address different aspects of this challenge but do not fully resolve it. Tensor-based methods provide a principled mechanism for preserving multidimensional structure and learning mode-specific latent factors \citep{zheng2019emotionmeter,wei2024statistical}, whereas multimodal learning methods enable flexible information exchange across heterogeneous data sources \citep{zadeh2017tensor, wu2024deep, peng2024learning}. However, the former typically express relationships among factors through decomposition or coupling mechanisms, while the latter often operate on modality-level embeddings or encode interactions implicitly through feature transformations, attention, or other fusion operations \citep{nagrani2021attention}. Consequently, neither paradigm directly treats tensor factors and their relationships as complementary components of the learned representation. A detailed review and positioning of these research streams is provided in Section~\ref{sec:literature}. 

This distinction motivates a relational perspective on multimodal tensor representation learning. Rather than viewing relationships only as intermediate computational mechanisms through which features are transformed, we consider whether relationships themselves can be learned, maintained, and incorporated into the final representation. Such a formulation is particularly relevant for multimodal tensor data, where individual dimensions have distinct structural identities and dependencies may arise both among dimensions within a modality and across heterogeneous modalities. The resulting challenge is to construct a representation that preserves factor-level structural information while maintaining relationship-level information in a scalable manner.

To address this challenge, we propose Relational Tensor Representation (RTR). The main contributions of this work are threefold. First, we formulate multimodal tensor representation learning from a relational perspective, in which dimension-specific latent factors and their cross-dimensional and cross-modality relationships are jointly maintained as explicit components of the learned representation. Second, we introduce learnable relational bridge matrices to represent and progressively update structural dependencies while preserving the identities of individual tensor dimensions and modalities. The resulting factor-level and relationship-level representations are integrated through a structured dimension-modality fusion mechanism, while the sparse bridge construction scales linearly with the total number of tensor dimensions. Third, we establish structural properties of the proposed RTR and conduct comprehensive evaluations through controlled simulations and two real-world case studies. The results demonstrate consistent improvements over compared baselines across different dependency structures, data scales, and degrees of modality heterogeneity. 

The remainder of this paper is organized as follows. 
Section~\ref{sec:literature} reviews the related literature and positions RTR with respect to existing representation learning methods. 
Section~\ref{sec:method} presents the proposed RTR framework, including dimension-specific tensor representation learning, relational bridge learning, and structured dimension-modality fusion. 
Section~\ref{sec:theory} establishes the structural and theoretical properties of the proposed framework. 
Section~\ref{sec:simulation} evaluates RTR through controlled simulation studies under different cross-dimensional and cross-modality dependency settings. 
Section~\ref{sec:real_world_experiments} presents two real-world case studies on multimodal neurophysiological data. 
Finally, Section~\ref{sec:conclusions} concludes the paper and discusses directions for future research.

\section{Related Work} 
\label{sec:literature}

The proposed Relational Tensor Representation (RTR) framework lies at the intersection of three closely related research streams: tensor representation learning, multimodal representation and fusion learning, and relationship-aware cross-modal learning. Tensor methods provide structured representations that preserve the multiway organization of individual modalities; multimodal learning methods focus on integrating complementary information across heterogeneous data sources; and relationship-aware approaches model interactions among representations. This section reviews these three streams and clarifies the methodological gap addressed by RTR.

\subsection{Tensor Representation Learning}

Tensor representation learning provides a natural framework for modeling high-dimensional data with multiway structure. Classical tensor decomposition methods, including CANDECOMP/PARAFAC (CP) and Tucker decomposition, approximate a high-order tensor through lower-dimensional latent factors while preserving its multilinear organization \citep{kolda2009tensor}. In contrast to vectorization, which collapses multiple dimensions into a single feature space, tensor representations retain the mode-specific structure of the data. This property is particularly useful when tensor modes have distinct physical or semantic interpretations, such as spatial locations, temporal intervals, physiological channels, or experimental conditions. Consequently, tensor models have been widely adopted in applications including signal processing, system monitoring, healthcare analytics, and spatiotemporal data analysis \citep{zhao2024change,sergin2025low}. 

A substantial body of research has extended classical tensor decompositions by incorporating structural assumptions and regularization mechanisms, including sparsity, smoothness, robustness, graph structure, and domain-specific constraints. These extensions improve the ability of tensor models to characterize complex data while retaining interpretable mode-specific latent factors. Another important direction is coupled tensor factorization, which extends tensor analysis from a single structured data source to multiple related sources. Coupled matrix and tensor factorization (CMTF), for example, jointly factorizes heterogeneous matrices and tensors by sharing latent factors along common modes \citep{acar2011all}. By jointly estimating these factors, CMTF enables information to be exchanged across heterogeneous sources and provides a principled framework for data fusion. Subsequent studies have further examined shared and source-specific structures in coupled factorizations, demonstrating both the benefits and limitations of factor sharing when the underlying data sources exhibit partially common latent structures. In biomedical applications, coupled tensor models have similarly been used to identify common and modality-specific patterns across physiological and imaging measurements \citep{zhao2023robust}.

More recently, tensor representation learning has increasingly incorporated deep neural networks to overcome the limited expressiveness of purely multilinear models. Deep tensor decomposition broadly extends conventional tensor factorization by introducing hierarchical or nonlinear mappings into the factorization process \citep{zhao2025deep}. These approaches can generally be viewed as following two complementary directions: linear deep tensor models recursively factorize latent factors to capture hierarchical representations at multiple levels of granularity, whereas nonlinear deep tensor models employ neural networks to characterize dependencies that cannot be adequately represented through multilinear operations alone. For example, DeepTensor parameterizes low-rank tensor components using deep generative networks and estimates them through a reconstruction objective \citep{saragadam2024deeptensor}. The implicit regularization of the neural parameterization enables the model to capture nonlinear signal structures while retaining a low-rank factorized representation. Collectively, these developments demonstrate an important transition from classical multilinear decomposition toward more expressive, hierarchical, and nonlinear tensor representations.

Despite this progress, the dominant representational unit in existing tensor models remains the factor. Classical decompositions learn mode-specific factors; structured tensor models impose additional constraints on these factors; coupled decompositions exchange information through shared or jointly estimated factors; and deep tensor models enhance the expressiveness of factors through hierarchical or nonlinear parameterizations. Relationships among factors are therefore typically expressed implicitly through multilinear reconstruction, shared-factor constraints, coupling structures, or neural transformations rather than being maintained as explicit components of the learned representation. This distinction is particularly important for multimodal tensor data. Consider latent factors associated with different dimensions or modalities. Existing tensor models may allow these factors to influence one another through a shared decomposition or joint optimization objective, but the relationship between them is generally not represented as an independent, learnable object that can subsequently participate in representation construction and downstream prediction. Consequently, information about how two factors are related may remain entangled within the factors themselves.

\subsection{Multimodal Representation and Fusion Learning}

Multimodal learning aims to combine complementary information from heterogeneous data sources to obtain more informative representations and improve downstream prediction. The field encompasses several closely related problems, including representation learning, alignment, fusion, translation, and co-learning \citep{baltrusaitis2019multimodal}. Within this broader landscape, multimodal fusion is commonly categorized according to the stage at which information is integrated, including early, intermediate (or representation-level), late, and hybrid fusion. Recent surveys have further examined these strategies in domain-specific settings. In healthcare, for example, multimodal learning has been organized around heterogeneous data sources, task-oriented modeling strategies, and emerging foundation models \citep{lin2025has}. For EEG-based emotion recognition, \citet{hatipogluyilmaz2026eeg} relate sensor characteristics and acquisition protocols to sensor-, feature-, and decision-level fusion strategies. These studies highlight that the appropriate fusion mechanism depends not only on the prediction task but also on the structural characteristics of the underlying modalities.

One important direction explicitly models higher-order interactions among modality representations. The Tensor Fusion Network (TFN), for example, constructs an outer-product-based representation that captures unimodal, bimodal, and trimodal interactions among modality embeddings \citep{zadeh2017tensor}. While this formulation provides an explicit representation of cross-modal feature combinations, the dimensionality of the resulting fusion tensor grows multiplicatively with the number and size of the modality representations. Low-rank Multimodal Fusion (LMF) addresses this computational challenge by exploiting a low-rank factorization of the multimodal fusion weights, thereby retaining multiplicative cross-modal interactions without explicitly constructing the full high-order fusion tensor \citep{liu2018efficient}. These methods demonstrate the value of modeling interactions beyond simple concatenation. However, the tensor structure in TFN and LMF primarily serves as a mechanism for organizing interactions among learned modality embeddings; it does not necessarily preserve the original mode-specific organization of tensor-valued measurements throughout the representation-learning process.

A second major direction uses attention mechanisms to learn cross-modal interactions adaptively. Transformer-based architectures have become particularly influential because attention allows the relevance of information exchanged across modalities to depend on the observed data. Self-attention captures dependencies within a representation sequence, while cross-attention enables one modality to selectively incorporate information from another. A representative example is the Multimodal Transformer (MulT), which employs directional pairwise cross-modal attention to model interactions between potentially unaligned multimodal sequences \citep{tsai2019multimodal}. Rather than requiring explicit temporal alignment, each modality can attend to relevant elements of another modality across different time steps. Subsequent studies have developed a broad range of Transformer-based multimodal architectures using cross-attention, co-attention, shared attention spaces, and hierarchical fusion mechanisms; these developments are reviewed by \citet{xu2023multimodal}. Collectively, these approaches establish attention as a powerful data-adaptive mechanism for information exchange across heterogeneous representations.

More recently, state-space models have provided an alternative mechanism for multimodal interaction, motivated in part by their computational efficiency for long sequences. Mamba and related selective state-space architectures model sequential dependencies through input-dependent state transitions, and their multimodal extensions introduce mechanisms for exchanging information across modality-specific state trajectories. Coupled Mamba, for example, introduces coupled state-space modeling in which the hidden-state evolution of one modality is influenced by states associated with neighboring modality chains \citep{li2024coupled}. Cross-modal information is therefore incorporated dynamically into the evolution of modality-specific hidden states. This represents a different interaction mechanism from attention-based fusion while similarly emphasizing adaptive information exchange across modalities. Beyond designing increasingly expressive fusion operators, recent research has emphasized how multimodal information is organized and retained within the learned representation. One line of work explicitly distinguishes shared and modality-specific information. Supervised multimodal fission learning (MMFL), for example, decomposes multimodal observations into components shared across all modalities, components shared by subsets of modalities, and modality-specific components while incorporating supervision from the downstream task \citep{mao2026supervised}.

Related work has also examined whether modality-specific and synergistic information survives the fusion process. CoMM aligns augmented joint multimodal representations through contrastive learning while accounting for shared, unique, and synergistic information \citep{dufumier2025align}. From a complementary perspective, \citet{chaudhuri2025collapse} study modality retention through cross-modal feature entanglement and rank bottlenecks, showing that multimodal representations may lose predictive information from individual modalities during fusion and proposing basis reallocation to mitigate this effect. These studies highlight a distinction that is increasingly important in multimodal representation learning: enabling interaction among modalities does not necessarily guarantee that the structural or predictive contributions of the individual modalities are retained.

Taken together, existing multimodal fusion methods differ substantially in their fundamental representation units. Early and intermediate fusion methods operate primarily on feature vectors or modality embeddings; TFN and LMF model multiplicative interactions among modality embeddings; Transformer-based approaches adapt modality sequences through attention; state-space approaches such as Coupled Mamba exchange information through evolving hidden states; and decomposition-based approaches organize information into shared and modality-specific latent components. These methods provide increasingly sophisticated mechanisms for multimodal interaction, but the internal multiway structure of tensor-valued modalities is typically compressed into modality-level representations before, or as part of, cross-modal fusion.

\subsection{Relationship-Aware and Cross-Modal Representation Learning}

Beyond learning strong representations for individual modalities, a growing body of research has emphasized the importance of explicitly characterizing relationships among heterogeneous representations. Depending on the modeling objective, relational information may describe semantic correspondence, feature compatibility, structural similarity, connectivity, or higher-order interaction patterns. Existing relationship-aware methods can be broadly viewed from three perspectives: relationships as interaction mechanisms for exchanging information, relationships as structural constraints for aligning representations, and relationships as explicit structures over which information is propagated.

A prominent example of the first perspective is cross-attention. Given representations from two modalities, cross-attention computes data-dependent compatibility scores that determine how information from one modality contributes to the representation of another. This mechanism has become fundamental to multimodal learning because it allows interactions to vary across samples and representation elements rather than relying on fixed correspondences. For example, MulT employs directional cross-modal attention to capture interactions between potentially unaligned multimodal sequences \citep{tsai2019multimodal}. More broadly, cross-attention and related co-attention mechanisms have been extensively adopted in vision--language learning, audio--visual modeling, multimodal retrieval, and biomedical applications. In these models, relational information primarily functions as an interaction mechanism: attention scores characterize compatibility between representations and are subsequently used to transform or update modality-specific features.

A second perspective treats relational structure as an alignment target. Rather than aligning individual features alone, these approaches seek to preserve or match the relationships among observations or latent representations. Cross-modality Representation Interactive Learning (CRIL), for example, learns modality-specific semantic representations and constructs self-attention matrices that characterize semantic relationships within each modality \citep{huang2023crossmodality}. Consistency among these relational structures is then encouraged across modalities, allowing relationship information from a semantically informative modality to guide representation learning in other modalities. Such methods demonstrate that relational patterns can provide supervision beyond pointwise feature correspondence and that transferring relational structure can improve multimodal alignment.

Graph-based multimodal learning provides a third perspective by representing relational structure explicitly through nodes and edges. Entities, sensors, features, or modality-specific representations can be represented as nodes, while their dependencies are encoded by graph edges; graph neural networks then propagate information according to the resulting connectivity structure \citep{peng2024learning}. Recent graph-based fusion methods have further increased the role of relational structure in multimodal representation learning. EGO fusion, for example, expands feature--relationship graphs and learns coefficients for combining graph expansions \citep{ding2025expansion}, while Graph4MM incorporates multi-hop neighborhood information through hop-diffused attention and employs a querying Transformer for cross-modal fusion \citep{ning2025graph4mm}. In these approaches, graph topology provides an explicit organizational structure through which multimodal information can interact.

Relational modeling is also increasingly important for physiological and sensor data, where dependencies among measurements may carry substantial predictive information. UAGCFNet, for example, models uncertainty in inter-sensor graph connectivity and combines attention-based fusion with contrastive calibration for physiological representation learning \citep{li2025uncertainty}. Such models are particularly relevant to multimodal biomedical data because physiological observations often exhibit structured dependencies across sensors, anatomical locations, or measurement channels. Graph-based representations provide a natural mechanism for capturing these dependencies when meaningful entities and connectivity structures can be defined.
 
Taken together, these studies demonstrate that relationships can enter multimodal learning in fundamentally different ways. A common characteristic, however, is that the relational quantities are typically tied to a particular computational role. For example, as attention weights for feature transformation, alignment structures used by an auxiliary objective, or graph edges governing message passing. This observation motivates an opportunity to formulate a unified relational representation learning framework that jointly models mode-specific tensor factors and their learned relationships, while capturing dependencies both among tensor dimensions within each modality and across different modalities. 

\subsection{Research Gap and Positioning of RTR}
RTR is designed to address the aforementioned research gap. Rather than replacing tensor factors with a unified multimodal embedding, RTR retains dimension-specific latent factors and augments them with pair-indexed relational bridge matrices. The resulting representation therefore contains two complementary forms of information: factor representations that preserve information associated with individual tensor dimensions and relational representations that characterize selected dependencies within and across modalities. The relational bridges are progressively updated using factor-level information and are explicitly retained for subsequent fusion. Consequently, RTR extends the conventional representational unit from factors or modality embeddings alone to a joint representation of tensor factors and their learned relational states. By integrating tensor-structured representations with explicit within- and cross-modality relationships, RTR provides a unified framework that connects tensor representation learning, multimodal fusion, and relationship-aware learning. The next section formalizes the proposed framework and describes how the tensor factors and relational bridges are constructed, updated, and integrated for downstream prediction.

\section{Methodology}
\label{sec:method}

This section presents the proposed RTR framework for representation learning from heterogeneous multimodal tensor data. As shown in Figure~\ref{fig:pipeline}, RTR consists of four major components. First, RTR decomposes each tensor modality into a collection of dimension-specific latent factors, preserving the structural semantics of individual tensor dimensions. Second, the resulting latent factors are mapped into a common embedding space while retaining their dimension and modality identity information. Third, RTR introduces two types of relational bridge matrices: dimension bridges that characterize within-modality (i.e., cross-dimensional) dependencies and modality bridges that characterize cross-modality dependencies. Finally, the learned factor-level representations and relationship-level bridge matrices are jointly fused into a compact representation for downstream prediction.  

\begin{figure}[htbp]
\centering
\includegraphics[width=1\textwidth]{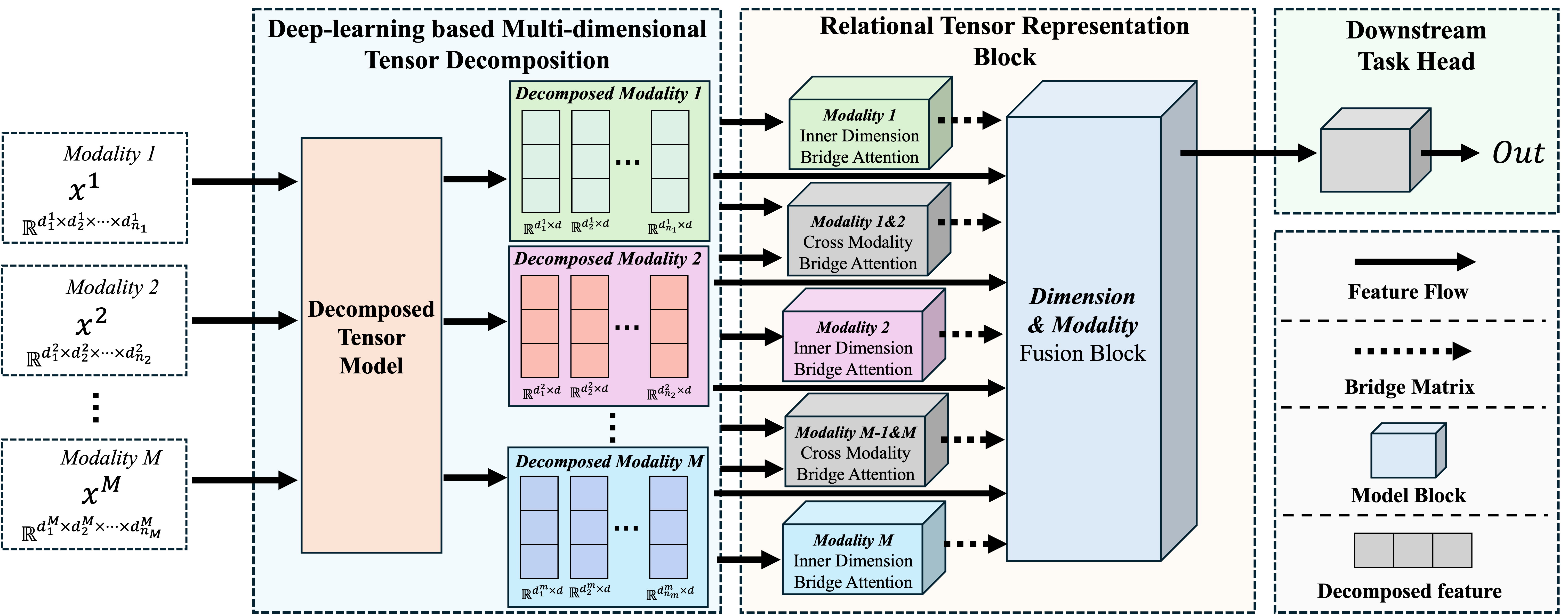}
\caption{Overview of the proposed RTR framework.}
\label{fig:pipeline}
\end{figure}

\subsection{Problem Formulation}
\label{sec:problem}

Consider a multimodal observation consisting of \(M\) heterogeneous tensor modalities,
\[
\mathcal X
=
\left\{
\mathcal X^{(1)},
\mathcal X^{(2)},
\ldots,
\mathcal X^{(M)}
\right\},
\]
where the $m$-th modality is denoted as a $K_m$-order tensor $\mathcal X^{(m)} \in \mathbb R^{d_1^{(m)} \times d_2^{(m)} \times \cdots \times d_{K_m}^{(m)}} (m=1,\ldots,M)$. Each tensor modality may have different tensor orders, dimension sizes, and structural semantics. The objective of RTR is to learn a unified representation
\begin{equation}
\mathcal Z
=
\Phi\!\left(
\mathcal U,
\mathcal B,
\mathcal C
\right),
\end{equation}
where \(\Phi\) represents the structured dimension--modality fusion operation, \(\mathcal U\) denotes the collection of dimension-specific latent factors extracted from all tensor modalities, \(\mathcal B\) denotes the collection of decoded within-modality bridge matrices, and \(\mathcal C\) denotes the collection of decoded cross-modality bridge matrices. Together, they provide a compact and informative representation \(\mathcal Z\) for downstream prediction tasks.  This formulation differs fundamentally from learning a single latent representation from each modality followed by interaction modeling. RTR does not treat cross-dimensional and cross-modality interactions as implicit attention scores only. Instead, it introduces relational bridge matrices as learnable relationship-level representations, enabling structural dependencies to be explicitly stored, updated, and fused. This design principle of RTR motivates the three stages developed below: dimension-specific tensor representation learning, relational bridge learning, and structured dimension-modality fusion.

\subsection{Dimension-Specific Tensor Representation Learning}
\label{sec:dimension_specific_representation}
RTR first extracts a collection of latent factors preserving structural information associated with each dimension. Specifically, each modality \(\mathcal X^{(m)}\) is unfolded along dimension $k$ as
\begin{equation}
\mathbf X_{(k)}^{(m)}
\in
\mathbb R^{
d_k^{(m)}
\times
\prod_{j\neq k}d_j^{(m)}
},
\end{equation}
where the rows of \(\mathbf X_{(k)}^{(m)}\) index the entries associated with dimension \(k\), whereas the columns contain the joint configurations across all remaining dimensions. Each unfolding therefore provides a distinct dimensional view of the tensor. We then obtain the corresponding latent factor through
\begin{equation}
\mathbf U_{m,k} = \sigma \left(\mathbf X_{(k)}^{(m)} \mathbf W_{m,k} \right),
\label{eq:dmtd}
\end{equation}
where 
\(\mathbf W_{m,k}\in
\mathbb R^{\left(\prod_{j\neq k}d_j^{(m)}\right)\times r}\)
is a learnable projection matrix, \(r\) is the latent feature dimension,
\(\mathbf U_{m,k}\in
\mathbb R^{d_k^{(m)}\times r}\)
is the resulting dimension-specific latent factor, and
\(\sigma(\cdot)\) denotes a nonlinear activation function. This projection compresses the joint information from all dimensions other than \(k\) into an \(r\)-dimensional latent space while retaining the \(d_k^{(m)}\) entries associated with the \(k\)-th dimension. Consequently, \(\mathbf U_{m,k}\) preserves the structural identity of the corresponding tensor dimension. Applying Eq.~\eqref{eq:dmtd} to every dimension of modality \(m\) produces
\begin{equation}
    \mathcal U_m = \left\{\mathbf U_{m,1}, \mathbf U_{m,2}, \ldots, \mathbf U_{m,K_m} \right\}.
\end{equation}

Repeating this procedure across all \(M\) modalities gives
\begin{equation} 
\mathcal U =
\left\{
\mathcal U_1,
\mathcal U_2,
\ldots,
\mathcal U_M
\right\}.
\end{equation}

The resulting collection of latent factors $\mathcal U$ is subsequently embedded into a common space for relational bridge learning. This representation preserves dimension-specific information within each modality and provides a consistent factor-level interface for modalities with different tensor orders and dimension sizes. 



\subsection{Relational Bridge Learning}
\label{sec:relational_bridge_learning}
The dimension-specific latent factors characterize the information associated with individual tensor dimensions but do not, by themselves, represent how those dimensions interact. RTR therefore introduces relational bridge matrices to represent and progressively refine structural dependencies among modalities. Each bridge matrix represents a structural relationship either between adjacent tensor dimensions within a modality or between adjacent modalities. This design allows RTR to explicitly store and update cross-dimensional and cross-modality dependencies instead of representing them only through transient attention scores.

\subsubsection{Common-Space Factor Embedding.}
\label{sec:embedded_factor_construction}

Since different tensor modalities may have different dimension sizes, RTR first maps all dimension-specific latent factors into a common embedding space. For each factor \(\mathbf U_{m,k}\), we apply a learnable projection along the dimension axis and transpose the result
\begin{equation}
\mathbf E_{m,k}
=
\left(
\mathbf P_{m,k}\mathbf U_{m,k}
\right)^{\top},
\label{eq:embedding}
\end{equation}
where \(\mathbf P_{m,k}\in\mathbb R^{d_e\times d_k^{(m)}}\) is a learnable dimension-alignment projection matrix and \(\mathbf E_{m,k}\in\mathbb R^{r\times d_e}\) is the embedded factor representation. Here, \(d_e\) denotes the common embedding dimension. This representation treats the latent feature dimension \(r\) as the sequence length and \(d_e\) as the embedding dimension, providing a consistent input format for relational bridge learning. To retain structural identity after projection, RTR adds three learnable embeddings as
\begin{equation}
\mathbf H_{m,k} = \mathbf E_{m,k} + \mathbf P^{\mathrm{intra}}_{m,k} + \mathbf P^{\mathrm{dim}}_k + \mathbf P^{\mathrm{mod}}_m,
\label{eq:pos_embedding}
\end{equation}
where \(\mathbf H_{m,k}\in\mathbb R^{r\times d_e}\) is the final embedded representation of the \(k\)-th dimension in modality \(m\). The intra-factor embedding \(\mathbf P^{\mathrm{intra}}_{m,k}\), dimension embedding \(\mathbf P^{\mathrm{dim}}_k\), and modality embedding \(\mathbf P^{\mathrm{mod}}_m\) are learnable embeddings in \(\mathbb R^{r\times d_e}\) that encode complementary structural information. Specifically, \(\mathbf P^{\mathrm{intra}}_{m,k}\) captures position-dependent variation within the embedded factor, while \(\mathbf P^{\mathrm{dim}}_k\) and \(\mathbf P^{\mathrm{mod}}_m\) encode the identities of the corresponding tensor dimension and modality, respectively. Together, these embeddings preserve the structural identity of each factor after all factors are projected into a common embedding space, allowing RTR to distinguish representations associated with different dimensions and modalities.

\subsubsection{Relational Bridge Construction.}
\label{sec:relational_bridge_construction}
After constructing embedded factors, RTR introduces two types of relational bridge matrices to represent structural dependencies within and across modalities. The first one is a within-modality dimension bridge. For modality \(m\), the dimension bridge \(\mathbf B_{m,k}\) connects two adjacent embedded factors \(\mathbf H_{m,k}\) and \(\mathbf H_{m,k+1}\), representing cross-dimensional dependencies inside each tensor modality. The second type is a cross-modality bridge. Each modality bridge \(\mathbf C_m\) connects two adjacent modalities \(m\) and \(m+1\), providing a structured representation of cross-modality dependencies without requiring fully connected interactions among all factors from different modalities. Each dimension bridge is initialized as follows
\begin{equation}
\mathbf B_{m,k}^{(0)} = \mathbf I_{d_e},
\label{eq:dim_bridge_init}
\end{equation}
where \(\mathbf B_{m,k}^{(0)}\in\mathbb R^{d_e\times d_e}\) is the initial dimension bridge between \(\mathbf H_{m,k}\) and \(\mathbf H_{m,k+1}\), and \(\mathbf I_{d_e}\in\mathbb R^{d_e\times d_e}\) denotes the identity matrix. Similarly, the modality bridge is initialized as 
\begin{equation}
\mathbf C_m^{(0)} = \mathbf I_{d_e}.
\label{eq:mod_bridge_init}
\end{equation}

The bridge matrices are initialized as identity matrices, providing a neutral starting point: before learning, no predefined structural dependency is imposed between tensor dimensions or modalities. During training, these bridge matrices are progressively updated and become data-adaptive relationship-level representations. This construction is intentionally sparse. Instead of assigning a bridge to every possible pair of factors, RTR uses adjacent dimension bridges within each modality and boundary bridges across adjacent modalities. This design preserves information flow across tensor dimensions and modalities while avoiding the quadratic growth of fully connected pairwise interaction modeling.

\subsubsection{Relational Bridge Updating.}
\label{sec:relational_bridge_updating}
RTR then updates each bridge matrix through an attention-based interaction mechanism. Embedded factors provide factor-level query representations, while bridge matrices provide relationship-level key and value representations. This separates the semantic content of each factor from the structural dependency encoded by the bridge. At the \(\ell\)-th bridge update layer, its key and value representations for each dimension bridge are computed as
\begin{equation}
\mathbf K_{m,k}^{\mathrm{dim},(\ell)}
=
\mathbf B_{m,k}^{(\ell)}\mathbf W_K,
\quad
\mathbf V_{m,k}^{\mathrm{dim},(\ell)}
=
\mathbf B_{m,k}^{(\ell)}\mathbf W_V,
\label{eq:dim_bridge_kv}
\end{equation}
where \(\mathbf W_K,\mathbf W_V\in\mathbb R^{d_e\times d_q}\) are learnable projection matrices, and \(\mathbf K_{m,k}^{\mathrm{dim},(\ell)}\) and \(\mathbf V_{m,k}^{\mathrm{dim},(\ell)}\) are the key and value representations induced by the dimension bridge \(\mathbf B_{m,k}^{(\ell)}\). Similarly, each modality bridge generates its key and value representations as follows
\begin{equation}
\mathbf K_{m}^{\mathrm{mod},(\ell)} = \mathbf C_m^{(\ell)}\mathbf W_K,
\quad
\mathbf V_{m}^{\mathrm{mod},(\ell)} = \mathbf C_m^{(\ell)}\mathbf W_V,
\label{eq:mod_bridge_kv}
\end{equation}
where \(\mathbf K_m^{\mathrm{mod},(\ell)}\) and \(\mathbf V_m^{\mathrm{mod},(\ell)}\) are the key and value representations induced by modality bridge \(\mathbf C_m^{(\ell)}\).

For within-modality dimension interaction, RTR first computes factor-level queries
\begin{equation}
\mathbf Q_{m,k}^{\mathrm{dim}}
=
\mathbf H_{m,k}\mathbf W_Q^{\mathrm{dim}},
\quad
k=1,\ldots,K_m,
\label{eq:dim_query_factor}
\end{equation}
where \(\mathbf W_Q^{\mathrm{dim}}\in\mathbb R^{d_e\times d_q}\) is a learnable query projection matrix and \(\mathbf Q_{m,k}^{\mathrm{dim}}\in\mathbb R^{r\times d_q}\) is the query representation of the \(k\)-th embedded factor in modality \(m\). The modality-internal mean query is then computed as
\begin{equation}
\bar{\mathbf Q}_m^{\mathrm{intra}}
=
\frac{1}{K_m}
\sum_{k=1}^{K_m}
\mathbf Q_{m,k}^{\mathrm{dim}}.
\label{eq:intra_mean_query}
\end{equation}
This mean query summarizes the overall factor-level information within modality \(m\) and is used to refine local dimension-level interactions.

For the dimension bridge between \(\mathbf H_{m,k}\) and \(\mathbf H_{m,k+1}\), the two factor-level queries are combined through an element-wise square-root interaction
\begin{equation}
\bar{\mathbf Q}_{m,k}^{\mathrm{dim}}
=
\sqrt{
\left|
\mathbf Q_{m,k}^{\mathrm{dim}}
\odot
\mathbf Q_{m,k+1}^{\mathrm{dim}}
\right|
+
\epsilon
},
\label{eq:dim_pair_query}
\end{equation}
where \(\odot\) denotes element-wise multiplication and \(\epsilon\) is a small constant for numerical stability. The first-stage dimension interaction feature map is computed as
\begin{equation}
\mathbf A_{m,k}^{\mathrm{dim},(\ell)}
=
\operatorname{Softmax}
\left(
\frac{
\bar{\mathbf Q}_{m,k}^{\mathrm{dim}}
\left(\mathbf K_{m,k}^{\mathrm{dim},(\ell)}\right)^{\top}
}
{\sqrt{d_q}}
\right)
\mathbf V_{m,k}^{\mathrm{dim},(\ell)}.
\label{eq:dim_attention}
\end{equation}

To incorporate broader within-modality context, RTR performs a second-stage refinement using the modality-internal mean query. The first-stage feature map is projected into keys and values
\begin{equation}
\widetilde{\mathbf K}_{m,k}^{\mathrm{dim},(\ell)}
=
\mathbf A_{m,k}^{\mathrm{dim},(\ell)}\widetilde{\mathbf W}_K,
\quad
\widetilde{\mathbf V}_{m,k}^{\mathrm{dim},(\ell)}
=
\mathbf A_{m,k}^{\mathrm{dim},(\ell)}\widetilde{\mathbf W}_V,
\label{eq:dim_refine_kv}
\end{equation}
where \(\widetilde{\mathbf W}_K\) and \(\widetilde{\mathbf W}_V\) are learnable projection matrices. The refined dimension interaction feature map is then computed as
\begin{equation}
\widetilde{\mathbf A}_{m,k}^{\mathrm{dim},(\ell)}
=
\operatorname{Softmax}
\left(
\frac{
\bar{\mathbf Q}_m^{\mathrm{intra}}
\left(\widetilde{\mathbf K}_{m,k}^{\mathrm{dim},(\ell)}\right)^{\top}
}
{\sqrt{d_q}}
\right)
\widetilde{\mathbf V}_{m,k}^{\mathrm{dim},(\ell)}.
\label{eq:dim_refine}
\end{equation}
The first stage captures localized interaction between adjacent dimension-specific factors, while the second stage allows the dimension bridge to incorporate modality-level context.

For cross-modality interaction, RTR uses the modality-internal mean queries from adjacent modalities. For the modality bridge between modality \(m\) and modality \(m+1\), the cross-modality query is constructed as
\begin{equation}
\bar{\mathbf Q}_{m}^{\mathrm{mod}}
=
\sqrt{
\left|
\bar{\mathbf Q}_m^{\mathrm{intra}}
\odot
\bar{\mathbf Q}_{m+1}^{\mathrm{intra}}
\right|
+
\epsilon
}.
\label{eq:mod_pair_query}
\end{equation}
The cross-modality interaction feature map is computed as 
\begin{equation}
\mathbf A_m^{\mathrm{mod},(\ell)}
=
\operatorname{Softmax}
\left(
\frac{
\bar{\mathbf Q}_{m}^{\mathrm{mod}}
\left(\mathbf K_{m}^{\mathrm{mod},(\ell)}\right)^{\top}
}
{\sqrt{d_q}}
\right)
\mathbf V_{m}^{\mathrm{mod},(\ell)}.
\label{eq:mod_attention}
\end{equation}
This design enables information propagation across modality boundaries without fully connecting all dimension-specific factors across modalities. The state feature of each bridge is then used to update the corresponding bridge matrix. For dimension bridges, we define
\begin{equation}
\mathbf S_{m,k}^{\mathrm{dim},(\ell)}
=
\widetilde{\mathbf A}_{m,k}^{\mathrm{dim},(\ell)}.
\label{eq:dim_bridge_state}
\end{equation}
The residual update for the dimension bridge is
\begin{equation}
\Delta \mathbf B_{m,k}^{(\ell)}
=
\operatorname{Reshape}
\left(
\operatorname{MLP}_B
\left(
\operatorname{vec}
\left(
\mathbf S_{m,k}^{\mathrm{dim},(\ell)}
\right)
\right)
\right),
\label{eq:dim_bridge_delta}
\end{equation}
\begin{equation}
\mathbf B_{m,k}^{(\ell+1)}
=
\operatorname{LN}
\left(
\mathbf B_{m,k}^{(\ell)}
+
\Delta \mathbf B_{m,k}^{(\ell)}
\right).
\label{eq:dim_bridge_update}
\end{equation}
For modality bridges, we define
\begin{equation}
\mathbf S_m^{\mathrm{mod},(\ell)}
=
\mathbf A_m^{\mathrm{mod},(\ell)}.
\label{eq:mod_bridge_state}
\end{equation}
The residual update for the modality bridge is computed as follows
\begin{equation}
\Delta \mathbf C_m^{(\ell)}
=
\operatorname{Reshape}
\left(
\operatorname{MLP}_C
\left(
\operatorname{vec}
\left(
\mathbf S_m^{\mathrm{mod},(\ell)}
\right)
\right)
\right),
\label{eq:mod_bridge_delta}
\end{equation}
\begin{equation}
\mathbf C_m^{(\ell+1)}
=
\operatorname{LN}
\left(
\mathbf C_m^{(\ell)}
+
\Delta \mathbf C_m^{(\ell)}
\right).
\label{eq:mod_bridge_update}
\end{equation}
Here, \(\operatorname{MLP}_B\) and \(\operatorname{MLP}_C\) map the interaction feature maps back to the bridge-matrix space, \(\operatorname{vec}(\cdot)\) denotes vectorization, \(\operatorname{Reshape}(\cdot)\) reshapes the output to a \(d_e\times d_e\) matrix, and \(\operatorname{LN}(\cdot)\) denotes layer normalization.

All query representations are generated from embedded factors and can be cached after their first computation. Across bridge update layers, the semantic factor representations remain fixed, while the dimension bridges and modality bridges are progressively updated. This separation preserves factor-level semantic identity while allowing relationship-level structures to evolve through learning.

\subsection{Dimension-Modality Fusion}
\label{sec:dimension_modality_fusion}

After relational bridge learning, RTR fuses dimension-specific latent factors and learned bridge matrices into a compact task representation. The goal of this stage is to preserve both factor-level information and relationship-level structural information. Specifically, for the \(k\)-th latent factor of modality \(m\), RTR first constructs a fixed-size latent self-correlation feature as
\begin{equation}
\mathbf G_{m,k}
= \mathbf U_{m,k}^{\top}\mathbf U_{m,k},
\label{eq:latent_gram}
\end{equation}
where \(\mathbf G_{m,k}\in\mathbb R^{r\times r}\) is the latent self-correlation feature of \(\mathbf U_{m,k}\). This operation maps each dimension-specific factor into the shared latent feature space and removes dependence on the original dimension size \(d_k^{(m)}\).

The learned dimension bridges \(\mathbf B_{m,k}^{(L)}\) and modality bridges \(\mathbf C_m^{(L)}\) are decoded into the latent fusion space. Let \(\widehat{\mathbf B}_{m,k}\in\mathbb R^{r\times r}\) denote the decoded dimension bridge between adjacent factors \(\mathbf U_{m,k}\) and \(\mathbf U_{m,k+1}\), and let \(\widehat{\mathbf C}_m\in\mathbb R^{r\times r}\) denote the decoded modality bridge between modalities \(m\) and \(m+1\). The modality-level fused representation is
\begin{equation}
\mathbf F_m
=
\mathbf G_{m,1}
\widehat{\mathbf B}_{m,1}
\mathbf G_{m,2}
\cdots
\widehat{\mathbf B}_{m,K_m-1}
\mathbf G_{m,K_m},
\label{eq:modality_fusion}
\end{equation}
where \(\mathbf F_m\in\mathbb R^{r\times r}\) is the fused representation of modality \(m\). The final cross-modality fused feature is then computed as
\begin{equation}
\mathbf F
=
\mathbf F_1
\widehat{\mathbf C}_1
\mathbf F_2
\cdots
\widehat{\mathbf C}_{M-1}
\mathbf F_M,
\label{eq:cross_modality_fusion}
\end{equation}
where \(\mathbf F\in\mathbb R^{r\times r}\) is the final cross-modality fused feature.

The fused feature and bridge matrices are stacked to form the final representation \(\mathcal Z\in\mathbb R^{N_z\times r\times r}\) 
\begin{equation}
\mathcal Z
=
\operatorname{Stack}
\left(
\mathbf F,
\{\widehat{\mathbf B}_{m,k}\},
\{\widehat{\mathbf C}_m\}
\right),
\label{eq:final_stack}
\end{equation}
where \(N_z\) is the total number of stacked fusion and bridge features. The downstream prediction head takes \(\mathcal Z\) as input, allowing it to use both the final fused representation and the learned relational structures. This fusion design has two advantages. First, it avoids discarding the bridge matrices after relational learning; instead, the learned relationship-level information is retained in the final representation. Second, by using fixed-size self-correlation features and decoded bridge matrices in the \(r\times r\) latent space, the final representation remains independent of the original tensor dimension sizes, making RTR applicable to heterogeneous multimodal tensors with different shapes.

\subsection{Model Training}
\label{sec:model_training}

RTR can be trained end-to-end for downstream supervised learning tasks. Given a training set \(\{(\mathcal X_i,y_i)\}_{i=1}^{n}\), the model first extracts dimension-specific factors, updates relational bridge matrices, constructs the fused representation \(\mathcal Z_i\), and then feeds \(\mathcal Z_i\) into a task-specific prediction head. For classification tasks, the training objective is
\begin{equation}
\argmin_{\Theta}
\frac{1}{n}
\sum_{i=1}^{n}
\mathcal L
\left(
f_{\Theta}(\mathcal Z_i),
y_i
\right)
+
\lambda
\Omega(\Theta),
\label{eq:training_objective}
\end{equation}
where \(f_{\Theta}(\cdot)\) denotes the prediction model, \(\mathcal L(\cdot,\cdot)\) is the classification loss, \(\Omega(\Theta)\) is a regularization term, and \(\lambda\) controls the regularization strength. The parameters \(\Theta\) include the dimension-specific projection matrices, embedding parameters, bridge update parameters, fusion parameters, and prediction head parameters.

Since all components are differentiable, RTR can be optimized using standard gradient-based methods. During training, the latent factors learn dimension-specific information, while the bridge matrices progressively adapt to the structural dependencies useful for the downstream task. This end-to-end training scheme allows RTR to learn task-relevant multimodal tensor representations without requiring hand-specified interaction rules.

\section{Structural Properties of RTR}
\label{sec:theory}

This section analyzes several structural properties of the proposed RTR. The goal is not to establish a generalization bound, but to state which properties follow from the specified architecture and under what conditions. The first two propositions formalize the roles of factors and bridges; the remaining propositions concern bridge count, storage, initialization, residual updating, and fusion. These structural properties do not by themselves establish recovery of ground-truth relationships or improved predictive performance.

\begin{proposition}[Bridge Matrices as Relationship-level Representations]\leavevmode\par
Let \(\mathbf H_{m,k}\) and \(\mathbf H_{m,k+1}\) denote two adjacent embedded dimension-specific factors within modality \(m\), and let \(\mathbf B_{m,k}^{(\ell)}\) denote their dimension bridge at layer \(\ell\). For fixed projection matrices, the bridge-induced key and value representations are
\[
\mathbf K_{m,k}^{\mathrm{dim},(\ell)}
=
\mathbf B_{m,k}^{(\ell)}\mathbf W_K,
\qquad
\mathbf V_{m,k}^{\mathrm{dim},(\ell)}
=
\mathbf B_{m,k}^{(\ell)}\mathbf W_V.
\]
For any perturbation \(\mathbf D\in\mathbb R^{d_e\times d_e}\), replacing \(\mathbf B_{m,k}^{(\ell)}\) by \(\mathbf B_{m,k}^{(\ell)}+\mathbf D\) changes these representations by \(\mathbf D\mathbf W_K\) and \(\mathbf D\mathbf W_V\), respectively. The same statement holds for a modality bridge \(\mathbf C_m^{(\ell)}\). Thus, a relationship-level representation is an explicit state assigned to a prescribed factor or modality connection and used to construct its attention inputs.
\end{proposition}

\begin{proof}
Subtracting the original key representation from the perturbed one gives \((\mathbf B_{m,k}^{(\ell)}+\mathbf D)\mathbf W_K-\mathbf B_{m,k}^{(\ell)}\mathbf W_K=\mathbf D\mathbf W_K\). The value representation satisfies the same identity with \(\mathbf W_V\). This is an algebraic property of the bridge parameterization; it does not require a semantic interpretation of the learned matrix entries.

RTR assigns \(\mathbf B_{m,k}^{(\ell)}\) to an adjacent dimension pair and \(\mathbf C_m^{(\ell)}\) to an adjacent modality pair. These indexed states are carried between update layers, rather than being defined solely as attention weights. This establishes the stated representation role. A nonzero perturbation need not change the keys or values if it is annihilated by the corresponding projection, and a change in either does not necessarily change the final prediction.
\end{proof}


\begin{proposition}[Separation of Factor-level and Relationship-level Representations]\leavevmode\par
Let \(\mathbf H_{m,k}\) be the embedded representation of the \(k\)-th tensor dimension in modality \(m\). Suppose the embedded factors are held fixed during the bridge-update layers, as specified in the relational bridge learning stage. Then these updates modify \(\mathbf B_{m,k}^{(\ell)}\) and \(\mathbf C_m^{(\ell)}\) without modifying \(\mathbf H_{m,k}\). Factor-level and relationship-level representations therefore occupy separate computational states.
\end{proposition}

\begin{proof}
The embedded factor \(\mathbf H_{m,k}\) is constructed from the dimension-specific latent factor \(\mathbf U_{m,k}\) together with intra-factor, dimension, and modality embeddings. Therefore, \(\mathbf H_{m,k}\) depends on the information contained in the corresponding tensor dimension and its structural identity.

On the other hand, the dimension bridge \(\mathbf B_{m,k}^{(\ell)}\) is assigned to the relationship between \(\mathbf H_{m,k}\) and \(\mathbf H_{m,k+1}\), while the modality bridge \(\mathbf C_m^{(\ell)}\) is assigned to the relationship between modalities \(m\) and \(m+1\). These bridge matrices are updated as separate states and generate key and value representations; the embedded factors remain fixed during this stage. Hence, the updates preserve the distinction between the two sets of variables. This separation does not imply statistical independence, disentanglement, or identifiability, since both sets of variables are derived from the same observations and learned jointly.
\end{proof}

The separation of factor and bridge variables provides an interface for diagnostic analysis and model design. If all tensor factors are simply concatenated into a single token sequence, the model may learn useful interactions, but it is difficult to distinguish whether performance comes from the quality of factor representations or from learned relationships among factors. In RTR, the bridge matrices provide an explicit structural layer between tensor factors, allowing factor perturbations and bridge perturbations to be examined separately.

\begin{proposition}[Linear Growth of Bridge Matrices]\leavevmode\par
Let \(K_m\) denote the number of dimension-specific latent factors in modality \(m\), with \(M\geq 1\) and \(K_m\geq 1\) for every modality, and let
\[
K_{\mathrm{tot}}
=
\sum_{m=1}^{M}K_m
\]
denote the total number of latent factors across all modalities. If RTR constructs exactly one bridge for each adjacent factor pair within a modality and one boundary bridge for each adjacent modality pair, with no additional bridges, then the total number of bridge matrices is
\[
N_B
=
\sum_{m=1}^{M}(K_m-1)+(M-1)
=
K_{\mathrm{tot}}-1.
\]
Therefore, the number of bridge matrices grows linearly with the total number of tensor dimensions across all modalities.
\end{proposition}

\begin{proof}
In modality \(m\), RTR connects adjacent dimension-specific factors. Since modality \(m\) has \(K_m\) factors, the number of within-modality bridges is \(K_m-1\). Summing over all modalities gives
\[
\sum_{m=1}^{M}(K_m-1)
\]
within-modality bridges. In addition, RTR constructs cross-modality boundary bridges between adjacent modalities. For \(M\) modalities, the number of such boundary bridges is \(M-1\). Therefore, the total number of bridge matrices is
\[
N_B
=
\sum_{m=1}^{M}(K_m-1)+(M-1)
=
\sum_{m=1}^{M}K_m-M+M-1
=
K_{\mathrm{tot}}-1.
\]
This completes the proof.
\end{proof}

This property shows that RTR avoids the quadratic growth associated with fully connected pairwise interaction modeling. If all \(K_{\mathrm{tot}}\) latent factors are connected with all possible pairwise relationships, the number of pairwise interactions is
\[
\frac{K_{\mathrm{tot}}(K_{\mathrm{tot}}-1)}{2}
=
O(K_{\mathrm{tot}}^2).
\]
In contrast, the bridge construction in RTR requires only \(K_{\mathrm{tot}}-1\) bridge matrices, which scales as \(O(K_{\mathrm{tot}})\). This linear growth is particularly important for multimodal tensor data because the total number of tensor dimensions can increase quickly as more modalities are added. The trade-off is that nonadjacent pairs do not receive dedicated bridge states. The result counts explicit bridges, not all computations performed by the model.

\begin{proposition}[Scalable Bridge State Representation]\leavevmode\par
Assume each bridge matrix has size \(d_e\times d_e\). Under the bridge construction in Proposition 3, for one sample at one update layer the total number of stored bridge-state entries is
\[
S_B
=
N_B d_e^2
=
(K_{\mathrm{tot}}-1)d_e^2.
\]
Thus, for fixed embedding dimension \(d_e\), the number of stored bridge-state entries grows linearly with the total number of tensor dimensions.
\end{proposition}

\begin{proof}
Each dimension bridge \(\mathbf B_{m,k}^{(\ell)}\) and each modality bridge \(\mathbf C_m^{(\ell)}\) contains \(d_e^2\) scalar entries. From Proposition 3, the total number of bridge matrices is \(N_B=K_{\mathrm{tot}}-1\). Therefore, the total number of stored bridge-state entries is
\[
S_B
=
N_B d_e^2
=
(K_{\mathrm{tot}}-1)d_e^2.
\]
For fixed \(d_e\), this quantity grows linearly with \(K_{\mathrm{tot}}\).
\end{proof}

The above result provides a more explicit view of the scalability of RTR. Although each bridge matrix is a dense relationship representation, the total number of such matrices is controlled by the sparse bridge graph. Therefore, RTR can use expressive matrix-valued relational representations without requiring all factor pairs to be connected. This design provides a balance between structural expressiveness and computational scalability. The count \(S_B\) is a state-storage measure, not the total number of trainable model parameters. It excludes projections, update networks, decoders, and the prediction head, as well as intermediate activations retained during training. Increasing \(d_e\) increases this state count quadratically.

\begin{proposition}[Neutrality of Identity Bridge Multiplication]\leavevmode\par
If each dimension bridge and modality bridge is initialized as an identity matrix
\[
\mathbf B_{m,k}^{(0)}
=
\mathbf I_{d_e},
\qquad
\mathbf C_m^{(0)}
=
\mathbf I_{d_e},
\]
then direct right multiplication by an initialized bridge leaves any compatible embedded factor unchanged: \(\mathbf H\mathbf B_{m,k}^{(0)}=\mathbf H\) and \(\mathbf H\mathbf C_m^{(0)}=\mathbf H\). This neutrality applies to the bridge matrix as a multiplier, not to the complete attention or fusion module.
\end{proposition}

\begin{proof}
For any row vector \(\mathbf h\in\mathbb R^{1\times d_e}\), multiplication by the identity matrix gives
\[
\mathbf h\mathbf I_{d_e}
=
\mathbf h.
\]
Similarly, for any matrix \(\mathbf H\in\mathbb R^{r\times d_e}\),
\[
\mathbf H\mathbf I_{d_e}
=
\mathbf H.
\]
These identities establish neutrality under direct multiplication. In the attention construction, however, an identity bridge gives \(\mathbf K=\mathbf W_K\) and \(\mathbf V=\mathbf W_V\); learned projections and attention can therefore act nontrivially even before the bridge is updated. The identity property does not extend automatically through attention, decoding, or normalization.
\end{proof}


\begin{proposition}[Residual Bridge Updating Carries Forward the Previous State]\leavevmode\par
Let the dimension bridge and modality bridge updates be
\[
\mathbf B_{m,k}^{(\ell+1)}
=
\operatorname{LN}
\left(
\mathbf B_{m,k}^{(\ell)}
+
\Delta \mathbf B_{m,k}^{(\ell)}
\right),
\]
\[
\mathbf C_m^{(\ell+1)}
=
\operatorname{LN}
\left(
\mathbf C_m^{(\ell)}
+
\Delta \mathbf C_m^{(\ell)}
\right).
\]
Then the pre-normalization update explicitly includes the previous bridge state. In particular, its change relative to \(\mathbf B_{m,k}^{(\ell)}\) has Frobenius norm \(\|\Delta\mathbf B_{m,k}^{(\ell)}\|_F\), with the analogous result for \(\mathbf C_m^{(\ell)}\). This is a statement about the residual update before normalization, not lossless retention of the previous state.
\end{proposition}

\begin{proof}
Each update equation contains an additive residual term. For example, the dimension bridge update contains
\[
\mathbf B_{m,k}^{(\ell)}
+
\Delta \mathbf B_{m,k}^{(\ell)}.
\]
Subtracting \(\mathbf B_{m,k}^{(\ell)}\) from this pre-normalization sum leaves \(\Delta\mathbf B_{m,k}^{(\ell)}\), which establishes the norm statement. The same argument applies to \(\mathbf C_m^{(\ell)}\). Layer normalization is not injective: a constant shift within a normalized group leaves its normalized values unchanged. Moreover, an update can cancel part or all of the previous state. Thus, the residual form establishes an explicit state-carrying path, but not recoverability of earlier states or an unconditional stability guarantee.
\end{proof}

This property helps explain why bridge matrices can be interpreted as evolving relationship-level representations. Each bridge is refined from its current state rather than reset at every layer. The update norms quantify the size of the pre-normalization changes and provide a useful diagnostic; information retention and stability after normalization require separate evaluation.

\begin{proposition}[Structured Fusion of Factors and Bridges]\leavevmode\par
Let
\[
\mathbf G_{m,k}
=
\mathbf U_{m,k}^{\top}\mathbf U_{m,k}
\]
denote the latent self-correlation feature of the \(k\)-th factor in modality \(m\). Let \(\widehat{\mathbf B}_{m,k}\) and \(\widehat{\mathbf C}_m\) denote the decoded dimension and modality bridges, respectively. The fusion rule
\[
\mathbf F_m
=
\mathbf G_{m,1}
\widehat{\mathbf B}_{m,1}
\mathbf G_{m,2}
\cdots
\widehat{\mathbf B}_{m,K_m-1}
\mathbf G_{m,K_m},
\]
\[
\mathbf F
=
\mathbf F_1
\widehat{\mathbf C}_1
\mathbf F_2
\cdots
\widehat{\mathbf C}_{M-1}
\mathbf F_M
\]
combines factor self-correlation features and decoded bridges in \(\mathbf F\). In the final stacked representation \(\mathcal Z\) of Eq.~\eqref{eq:final_stack}, each decoded bridge is also retained as a separate block and can be recovered by selecting that block. This recovery statement concerns the decoded bridges, not the original factors or input tensors.
\end{proposition}

\begin{proof}
Each \(\mathbf G_{m,k}\) is constructed from the latent factor \(\mathbf U_{m,k}\), and therefore encodes factor-level information in the shared latent space. Each decoded dimension bridge \(\widehat{\mathbf B}_{m,k}\) is assigned to the connection between adjacent tensor factors within modality \(m\), while each decoded modality bridge \(\widehat{\mathbf C}_m\) is assigned to the connection between adjacent modalities.

The fused representation \(\mathbf F_m\) is obtained by composing factor-level self-correlation features with decoded dimension bridges. Therefore, \(\mathbf F_m\) depends on both the latent factors and the learned within-modality relationships. Similarly, \(\mathbf F\) is obtained by composing modality-level representations with decoded modality bridges, and therefore depends on both modality-level features and cross-modality relationships. The stack in Eq.~\eqref{eq:final_stack} appends the decoded bridges themselves to \(\mathbf F\), so selecting their designated blocks recovers them exactly. No invertibility of the fusion products is needed for this block-selection property.
\end{proof}


\begin{proposition}[Shape Invariance of the Final Latent Representation]\leavevmode\par
For each dimension specific factor \(\mathbf U_{m,k}\in\mathbb R^{d_k^{(m)}\times r}\), the latent self-correlation feature
\[
\mathbf G_{m,k}
=
\mathbf U_{m,k}^{\top}\mathbf U_{m,k}
\]
has size \(r\times r\), independent of the original tensor dimension size \(d_k^{(m)}\). If all decoded bridges are also represented in \(\mathbb R^{r\times r}\), then for fixed \(M\), \(\{K_m\}\), and \(r\), the shapes of \(\mathbf F\) and \(\mathcal Z\) are independent of the original tensor dimension sizes. In particular, \(\mathbf F\in\mathbb R^{r\times r}\) and \(\mathcal Z\in\mathbb R^{(N_B+1)\times r\times r}\) for the stack in Eq.~\eqref{eq:final_stack}.
\end{proposition}

\begin{proof}
Since \(\mathbf U_{m,k}\in\mathbb R^{d_k^{(m)}\times r}\), its transpose satisfies \(\mathbf U_{m,k}^{\top}\in\mathbb R^{r\times d_k^{(m)}}\). Therefore,
\[
\mathbf G_{m,k}
=
\mathbf U_{m,k}^{\top}\mathbf U_{m,k}
\in
\mathbb R^{r\times r}.
\]
The original dimension size \(d_k^{(m)}\) does not appear in the output shape of this matrix multiplication. If each decoded bridge matrix \(\widehat{\mathbf B}_{m,k}\) and \(\widehat{\mathbf C}_m\) also lies in \(\mathbb R^{r\times r}\), then all terms in the fusion products are \(r\times r\) matrices. Therefore, the shapes of \(\mathbf F_m\), \(\mathbf F\), and \(\mathcal Z\) depend on \(r\) and the number of stacked matrices, but not on \(d_k^{(m)}\). Their numerical values can still depend on the input tensors and their dimensions.
\end{proof}

This property is useful for heterogeneous multimodal tensor data. Different modalities may have different dimension sizes, but after dimension-specific representation and latent self-correlation construction, all modality representations are mapped into a common latent matrix space. As a result, the downstream prediction head can operate on a consistent representation format even when the input modalities have different tensor shapes. This is an output-format property, not a guarantee that one trained model accepts arbitrary new input shapes: the input-dependent projection dimensions must still be compatible.

Together, these properties describe the structural rationale of RTR. Bridges provide indexed relational states distinct from factor variables; the specified sparse topology bounds their count and state-storage requirements; and the update and fusion rules specify how those states are carried forward and made available to the prediction head. The resulting trade-offs concern connectivity, representation width, update size, and the information exposed at readout. Their effects on predictive performance and the meaning of the learned relationships require empirical assessment. This distinction keeps architectural guarantees separate from model insights to be investigated in simulations and real-world cases.

\section{Simulation Study}
\label{sec:simulation}
This section evaluates the proposed RTR through simulation studies under four scenarios. All scenarios use a common family of latent class geometries while varying the contributions of structured and modality-specific variation. The objective is to examine the predictive behavior of the full model under these different mixtures and relate the observations to the representation roles described in Section~\ref{sec:theory}. These comparisons do not assume identical classification difficulty in the observed tensor space.

\subsection{Simulation Objective} 
We construct a synthetic binary classification problem with two third-order tensor modalities. For each sample $i$ $(i=1,\ldots,N)$, the observed data are $(\mathcal{X}_i^{(1)},  \mathcal{X}_i^{(2)})$ with a class label $y_i \in \{0,1\}$. For each simulation case, we generate \(N=3000\) samples with $\mathcal{X}_i^{(1)} \in \mathbb R^{20\times15\times10}$ and $\mathcal{X}_i^{(2)} \in \mathbb R^{20\times15\times10}$ under balanced class proportions. The data-generating process is governed by two structural parameters: $\alpha\in[0,1]$, which controls the strength of within-modality cross-dimensional dependency, and $\beta\in[0,1]$, which controls the strength of cross-modality dependency. Here, \(\alpha\) and \(\beta\) are generation parameters, rather than measured dependence coefficients of the final observed tensors. The four scenarios retain the same type of latent class geometry while changing how its signal is expressed in the observations. This design addresses two questions: whether RTR remains competitive across the prescribed within- and cross-modality mixtures, and whether the relative performance of the comparison methods changes with the scenario. 

\subsection{Multimodal Tensor Generation} 
\label{sec:simulation_generation}
To generate each sample $(\mathcal{X}_i^{(1)}, \mathcal{X}_i^{(2)})$, we first generate a shared latent representation
\(\mathbf z_i^{\mathrm{shared}}\in\mathbb R^r\) from a Taiji-like two-moon geometry. The two classes correspond to two nonlinear latent manifolds, so that \(\mathbf z_i^{\mathrm{shared}}\) serves as the common class-discriminative source for both modalities. We set $r=5$: the first two coordinates follow the two-moon geometry with Gaussian jitter of standard deviation $0.08$, coordinatewise centering and unit-variance scaling, and a final multiplier $\delta=0.8$; the remaining coordinates are independent $N(0,0.35^2)$ variables ($\tau=1$). To control cross-modality dependency, each modality combines the shared latent representation with an independent private perturbation:
$\mathbf z_i^{(m)}=\beta \mathbf z_i^{\mathrm{shared}}+\sqrt{1-\beta^2}\,
\mathbf z_{i,\mathrm{private}}^{(m)}$, where $\mathbf z_{i,\mathrm{private}}^{(m)}\sim N(\mathbf0,\mathbf I_r)$. The parameter \(\beta\) controls the cross-modality dependency and a smaller \(\beta\) increases the contribution of modality-specific variation. The structured core places $\mathbf z_i^{(m)}$ on its diagonal and adds the rank-one term $(0.15/\sqrt r)\tanh(\mathbf z_i^{(m)})^{\otimes3}$. For the weak core, entry $(a,b,c)$ is $(u_{1,a}+u_{2,b}+u_{3,c})/\sqrt{3r}+0.7\xi_{abc}$, where $\mathbf u_k=\mathbf R_{m,k}^{\top}\mathbf z_i^{(m)}$ and $\xi_{abc}$ are independent $N(0,1)$ variables; sample and modality indices are suppressed here for brevity. Each modality-specific latent core is subsequently embedded into a third-order tensor. The complete tensor-generation mechanism is written as
\begin{equation}
\begin{aligned}
\mathcal Y_i^{(m)}
&=\big\{\alpha\mathcal G_{\mathrm{structured},i}^{(m)}
+(1-\alpha)\mathcal G_{\mathrm{weak},i}^{(m)}\big\}
\times_1\mathbf A_{m,1}\times_2\mathbf A_{m,2}\times_3\mathbf A_{m,3},\\
\mathcal X_i^{(m)}
&=\Big[\operatorname{Normalize}\!\big(
\operatorname{Normalize}(\mathcal Y^{(m)})
+\lambda\operatorname{Normalize}(\mathcal C^{(m)})\big)\Big]_i
+\sigma\mathcal E_i^{(m)},
\end{aligned}
\label{eq:simulation_observed_tensor}
\end{equation}
where \(\alpha\) controls the relative contributions of the structured and weak core components. \(\mathbf A_{m,1}\in\mathbb R^{20\times r}\), \(\mathbf A_{m,2}\in\mathbb R^{15\times r}\), and \(\mathbf A_{m,3}\in\mathbb R^{10\times r}\) map the latent core to the three observed tensor dimensions, where \(\times_k\) denotes the tensor--matrix product along dimension \(k\). The $\mathbf A_{m,k}$ and $\mathbf R_{m,k}\in\mathbb R^{r\times r}$ matrices have independent standard Gaussian entries before columnwise unit-$\ell_2$ normalization and are fixed across samples within each scenario. In Eq.~\eqref{eq:simulation_observed_tensor}, $\mathcal Y^{(m)}$ and $\mathcal C^{(m)}$ stack all samples, $[\cdot]_i$ selects sample $i$, and $\operatorname{Normalize}(\mathcal T)=\mathcal T/\max\{\operatorname{sd}_{\mathrm{all}}(\mathcal T),10^{-6}\}$ divides by the population standard deviation over that modality's full stack without centering. The carrier tensor \(\mathcal C_i^{(m)}\in\mathbb R^{20\times15\times10}\) provides an additional smooth expression of the shared nonlinear latent geometry. The two modalities use slightly different carrier patterns so that they remain related but not identical. The tensor \(\mathcal E_i^{(m)}\in\mathbb R^{20\times15\times10}\) contains independent standard Gaussian noise, and \(\lambda\) and \(\sigma\) control the carrier strength and noise level, respectively. All scenarios use $\lambda=20$ and $\sigma=0.05$; the combined signal is rescaled before observation noise is added. Separately, classifier inputs are standardized coordinatewise using only training-set means and sample standard deviations (floored at $10^{-6}$), with the same transformation applied to validation and test data.

\subsection{Structural Dependency Scenarios}
\label{sec:simulation_scenarios}
We consider four simulation scenarios by combining weak and strong levels of cross-dimensional and cross-modality dependencies, controlled by \(\alpha\) and \(\beta\), as summarized in Table~\ref{tab:simulation_scenarios}. Scenario 1 represents weak dependencies at both levels (\(\alpha=0.2\), \(\beta=0.2\)), whereas Scenario 2 increases the cross-dimensional dependency while maintaining weak cross-modality dependency (\(\alpha=0.8\), \(\beta=0.2\)). Scenario 3 considers the opposite setting, with weak cross-dimensional but strong cross-modality dependency (\(\alpha=0.2\), \(\beta=0.8\)). Finally, Scenario 4 represents strong dependencies at both levels (\(\alpha=0.8\), \(\beta=0.8\)). The weak and strong labels identify the prescribed low and high parameter settings. They do not imply the absence or presence of all shared information: the carrier term in Eq.~\eqref{eq:simulation_observed_tensor} also expresses the shared latent geometry. Thus, a low value of \(\beta\) should not be interpreted as making the two observed modalities independent.
\begin{table}[htbp]
\centering
\caption{Simulation scenarios under different levels of within-modality and cross-modality dependency.}
\label{tab:simulation_scenarios}
\small
\setlength{\tabcolsep}{6pt}
\begin{tabular}{lcc}
\toprule
Scenarios
& \begin{tabular}{c} Cross-dimensional \\ dependency \(\alpha\) \end{tabular}
& \begin{tabular}{c} Cross-modality \\ dependency \(\beta\) \end{tabular} \\
\midrule
Scenario 1: Weak -- Weak      & 0.2 & 0.2 \\
Scenario 2: Strong -- Weak    & 0.8 & 0.2 \\
Scenario 3: Weak -- Strong    & 0.2 & 0.8 \\
Scenario 4: Strong -- Strong  & 0.8 & 0.8 \\
\bottomrule
\end{tabular}
\end{table}

Figure~\ref{fig:simulation_case_difference} illustrates the intended distinction among the four scenarios, with one column per scenario. The upper row depicts the dimension-level setting associated with \(\alpha\), and the lower row depicts the modality-level setting associated with \(\beta\). Blue and red point clouds distinguish two tensor dimensions; purple and orange clouds distinguish the two modalities. Gray lines connect corresponding observations. The separated and overlapping clouds provide a visual guide to the weak and strong design labels. In particular, Scenarios 2 and 3 contrast the two levels, whereas Scenarios 1 and 4 use matching levels. The figure is an illustration of the scenario design, not a quantitative dependence estimate or an evaluation of the relational states learned by RTR.
\begin{figure}[htbp]
\centering
\includegraphics[width=\textwidth]{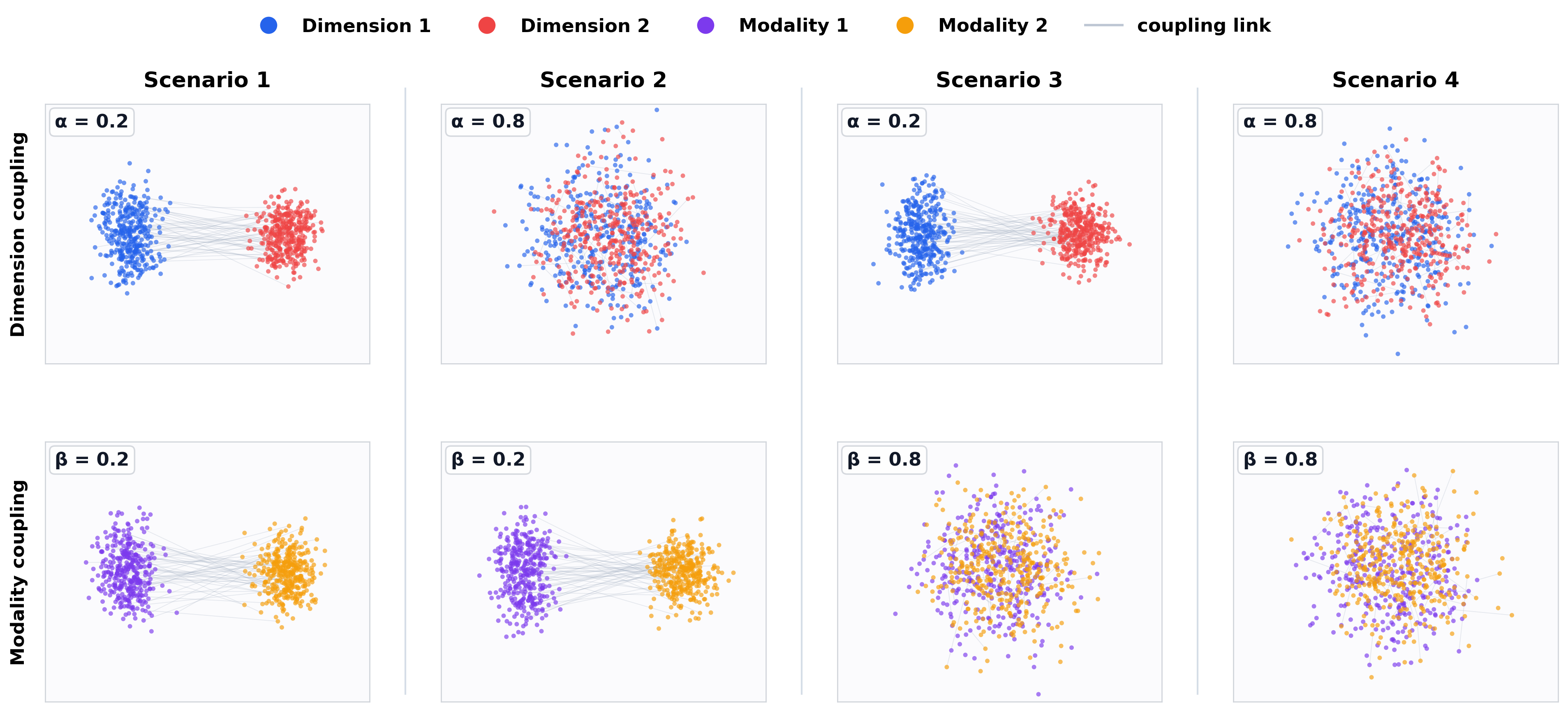}
\caption{Illustrative dependency patterns under the four simulation scenarios.} 
\label{fig:simulation_case_difference}
\end{figure}

\subsection{Simulation Results}
\label{sec:simulation_downstream}
For each simulation scenario, the data are randomly divided into $70\%$ training, $10\%$ validation, and $20\%$ testing sets, and each experiment is repeated four times with different random seeds. All compared methods use the same protocol. Performance is reported as mean \(\pm\) standard deviation. The compared methods include Support Vector Machine (SVM) \citep{cervantes2020comprehensive}, Multilayer Perceptron (MLP) \citep{liu2022we}, Tensor Fusion \citep{zadeh2017tensor}, Coupled Matrix and Tensor Factorization (CMTF) \citep{acar2011all}, Transformer \citep{vaswani2017attention}, Multimodal Transformer \citep{xu2023multimodal}, and Coupled Mamba \citep{li2024coupled}. SVM and MLP provide a conventional vectorized classification baseline. Modality1 MLP and Modality2 MLP evaluate the discriminative information contained in each individual tensor modality, whereas Tensor Fusion combines the two modalities for classification. CMTF represents a tensor factorization-based multimodal learning method. Transformer and Multimodal Transformer provide attention-based baselines, with the latter explicitly modeling interactions across modalities. Coupled Mamba provides a state-space-model-based multimodal baseline for capturing long-range dependencies and exchanging information between modalities. Together, these baselines cover vectorized classification, single-modality learning, feature-level multimodal fusion, tensor-based multimodal fusion, attention-based multimodal fusion, and state-space modeling.  
\begin{table*}[htbp]
\centering
\caption{Simulation performance under four scenarios. Results are reported as mean $\pm$ standard deviation over four independent repeats. Bold values indicate the best performance in each setting.}
\label{tab:simulation_results}

\scriptsize
\setlength{\tabcolsep}{3pt}
\renewcommand{\arraystretch}{1.08}

\resizebox{\textwidth}{!}{
\begin{tabular}{lcccccccc}
\toprule

\multicolumn{9}{l}{
\textbf{Panel A: Weak cross-modality dependency ($\beta=0.2$)}
}
\\[2pt]

\multirow{2}{*}{Method}
&
\multicolumn{4}{c}{Scenario 1: Weak dimension ($\alpha=0.2$)}
&
\multicolumn{4}{c}{Scenario 2: Strong dimension ($\alpha=0.8$)}
\\
\cmidrule(lr){2-5}
\cmidrule(lr){6-9}

&
Acc. & Precision & Recall & Macro-F1
&
Acc. & Precision & Recall & Macro-F1
\\
\midrule

SVM
& \(0.8653 \pm 0.0252\) & \(0.8655 \pm 0.0252\) & \(0.8653 \pm 0.0252\) & \(0.8653 \pm 0.0252\)
& \(0.8670 \pm 0.0406\) & \(0.8671 \pm 0.0406\) & \(0.8670 \pm 0.0406\) & \(0.8669 \pm 0.0406\)
\\

Modality1 MLP
& \(0.9069 \pm 0.0171\) & \(0.9073 \pm 0.0176\) & \(0.9069 \pm 0.0171\) & \(0.9069 \pm 0.0171\)
& \(0.8980 \pm 0.0248\) & \(0.8981 \pm 0.0247\) & \(0.8980 \pm 0.0248\) & \(0.8980 \pm 0.0248\)
\\

Modality2 MLP
& \(0.9080 \pm 0.0143\) & \(0.9084 \pm 0.0146\) & \(0.9080 \pm 0.0143\) & \(0.9080 \pm 0.0143\)
& \(0.8952 \pm 0.0289\) & \(0.8955 \pm 0.0287\) & \(0.8952 \pm 0.0289\) & \(0.8952 \pm 0.0289\)
\\

Tensor Fusion
& \(0.9030 \pm 0.0158\) & \(0.9032 \pm 0.0160\) & \(0.9030 \pm 0.0158\) & \(0.9030 \pm 0.0158\)
& \(0.8958 \pm 0.0311\) & \(0.8960 \pm 0.0310\) & \(0.8958 \pm 0.0311\) & \(0.8958 \pm 0.0311\)
\\

CMTF
& \(0.9113 \pm 0.0272\) & \(0.9115 \pm 0.0274\) & \(0.9113 \pm 0.0272\) & \(0.9113 \pm 0.0272\)
& \(0.8975 \pm 0.0265\) & \(0.8977 \pm 0.0264\) & \(0.8975 \pm 0.0265\) & \(0.8974 \pm 0.0265\)
\\

Transformer
& \(0.9063 \pm 0.0114\) & \(0.9067 \pm 0.0118\) & \(0.9063 \pm 0.0114\) & \(0.9063 \pm 0.0113\)
& \(0.8941 \pm 0.0283\) & \(0.8943 \pm 0.0280\) & \(0.8941 \pm 0.0283\) & \(0.8941 \pm 0.0283\)
\\

Multimodal Transformer
& \(0.9390 \pm 0.0224\) & \(0.9391 \pm 0.0225\) & \(0.9390 \pm 0.0224\) & \(0.9390 \pm 0.0224\)
& \(0.9296 \pm 0.0387\) & \(0.9300 \pm 0.0388\) & \(0.9296 \pm 0.0387\) & \(0.9296 \pm 0.0387\)
\\

Coupled Mamba
& \(0.9446 \pm 0.0127\) & \(0.9449 \pm 0.0130\) & \(0.9446 \pm 0.0127\) & \(0.9446 \pm 0.0127\)
& \(0.9346 \pm 0.0243\) & \(0.9350 \pm 0.0238\) & \(0.9346 \pm 0.0243\) & \(0.9346 \pm 0.0244\)
\\

\midrule
RTR
& \(\mathbf{0.9979 \pm 0.0013}\) & \(\mathbf{0.9979 \pm 0.0013}\)
& \(\mathbf{0.9979 \pm 0.0013}\) & \(\mathbf{0.9979 \pm 0.0013}\)
& \(\mathbf{0.9889 \pm 0.0101}\) & \(\mathbf{0.9882 \pm 0.0104}\)
& \(\mathbf{0.9871 \pm 0.0102}\) & \(\mathbf{0.9890 \pm 0.0100}\)
\\

\addlinespace[5pt]
\midrule
\addlinespace[2pt]

\multicolumn{9}{l}{
\textbf{Panel B: Strong cross-modality dependency ($\beta=0.8$)}
}
\\[2pt]

\multirow{2}{*}{Method}
&
\multicolumn{4}{c}{Scenario 3: Weak dimension ($\alpha=0.2$)}
&
\multicolumn{4}{c}{Scenario 4: Strong dimension ($\alpha=0.8$)}
\\
\cmidrule(lr){2-5}
\cmidrule(lr){6-9}

&
Acc. & Precision & Recall & Macro-F1
&
Acc. & Precision & Recall & Macro-F1
\\
\midrule

SVM
& \(0.8697 \pm 0.0064\) & \(0.8701 \pm 0.0062\) & \(0.8697 \pm 0.0063\) & \(0.8697 \pm 0.0064\)
& \(0.8548 \pm 0.0271\) & \(0.8552 \pm 0.0270\) & \(0.8547 \pm 0.0271\) & \(0.8547 \pm 0.0272\)
\\

Modality1 MLP
& \(0.8975 \pm 0.0105\) & \(0.8978 \pm 0.0106\) & \(0.8974 \pm 0.0104\) & \(0.8974 \pm 0.0104\)
& \(0.8747 \pm 0.0246\) & \(0.8752 \pm 0.0248\) & \(0.8747 \pm 0.0246\) & \(0.8747 \pm 0.0246\)
\\

Modality2 MLP
& \(0.8969 \pm 0.0090\) & \(0.8971 \pm 0.0089\) & \(0.8969 \pm 0.0089\) & \(0.8969 \pm 0.0090\)
& \(0.8758 \pm 0.0207\) & \(0.8764 \pm 0.0208\) & \(0.8758 \pm 0.0207\) & \(0.8758 \pm 0.0207\)
\\

Tensor Fusion
& \(0.8975 \pm 0.0098\) & \(0.8979 \pm 0.0096\) & \(0.8974 \pm 0.0098\) & \(0.8974 \pm 0.0098\)
& \(0.8742 \pm 0.0269\) & \(0.8744 \pm 0.0271\) & \(0.8742 \pm 0.0269\) & \(0.8742 \pm 0.0269\)
\\

CMTF
& \(0.8980 \pm 0.0124\) & \(0.8981 \pm 0.0125\) & \(0.8980 \pm 0.0124\) & \(0.8980 \pm 0.0124\)
& \(0.8936 \pm 0.0316\) & \(0.8937 \pm 0.0316\) & \(0.8936 \pm 0.0315\) & \(0.8936 \pm 0.0316\)
\\

Transformer
& \(0.8980 \pm 0.0110\) & \(0.8983 \pm 0.0112\) & \(0.8980 \pm 0.0110\) & \(0.8980 \pm 0.0110\)
& \(0.8731 \pm 0.0217\) & \(0.8733 \pm 0.0219\) & \(0.8730 \pm 0.0217\) & \(0.8730 \pm 0.0217\)
\\

Multimodal Transformer
& \(0.9568 \pm 0.0117\) & \(0.9573 \pm 0.0121\) & \(0.9568 \pm 0.0117\) & \(0.9568 \pm 0.0117\)
& \(0.9279 \pm 0.0131\) & \(0.9281 \pm 0.0132\) & \(0.9279 \pm 0.0131\) & \(0.9279 \pm 0.0131\)
\\

Coupled Mamba
& \(0.9385 \pm 0.0202\) & \(0.9390 \pm 0.0199\) & \(0.9385 \pm 0.0202\) & \(0.9385 \pm 0.0202\)
& \(0.9252 \pm 0.0186\) & \(0.9256 \pm 0.0182\) & \(0.9251 \pm 0.0186\) & \(0.9251 \pm 0.0186\)
\\

\midrule
RTR
& \(\mathbf{0.9856 \pm 0.0144}\) & \(\mathbf{0.9851 \pm 0.0121}\)
& \(\mathbf{0.9855 \pm 0.0101}\) & \(\mathbf{0.9855 \pm 0.0101}\)
& \(\mathbf{0.9917 \pm 0.0053}\) & \(\mathbf{0.9915 \pm 0.0058}\)
& \(\mathbf{0.9911 \pm 0.0058}\) & \(\mathbf{0.9912 \pm 0.0052}\)
\\

\bottomrule
\end{tabular}
}
\end{table*}

Table~\ref{tab:simulation_results} reports the classification performance across all four scenarios. RTR has the highest reported mean on each evaluation metric in each scenario. Agreement in the ranking across accuracy, precision, recall, and macro-F1 indicates that the comparison is not driven by a single reported metric. In scenario 1, RTR achieves an accuracy and macro-F1 of \(0.9979\), compared with \(0.9446\) for Coupled Mamba, the strongest baseline in this case. This result shows strong reported classification performance in the joint low-parameter setting, which still includes the shared carrier signal. Scenario 2 strengthens the within-modality cross-dimensional dependency while retaining weak cross-modality dependency. RTR achieves \(0.9889\) accuracy and \(0.9890\) macro-F1, whereas Coupled Mamba achieves \(0.9346\) on both metrics. The comparison favors RTR within this setting. However, its mean accuracy is lower than in Scenario 1, so the result does not support a monotonic benefit from increasing the structured-core weight. Scenario 3 considers weak cross-dimensional dependency and strong cross-modality dependency. Multimodal Transformer becomes the strongest baseline, achieving \(0.9568\) accuracy and macro-F1. The change in the leading baseline illustrates that the relative performance of the comparison methods is scenario-dependent. Nevertheless, RTR achieves \(0.9856\) accuracy and \(0.9855\) macro-F1, which is consistent with the usefulness of jointly processing dimension-level and modality-level information in this setting. Scenario 4 contains strong cross-dimensional and cross-modality dependencies. RTR achieves the best accuracy and macro-F1, compared with \(0.9279\) for Multimodal Transformer, the strongest baseline in this case. Together with Scenario 3, this result places the full RTR model first in both settings of Panel B.

The baseline results offer two further observations. First, SVM achieves approximately \(0.85\)--\(0.87\) accuracy and the single-modality MLPs perform better, showing that each modality contains predictive information. The stronger multimodal methods improve on these single-modality models in the reported comparisons. This is evidence about the evaluated pipelines, not proof that either modality alone lacks the information needed for classification. Second, the leading baseline changes from Coupled Mamba in Scenarios 1 and 2 to Multimodal Transformer in Scenarios 3 and 4. Tensor Fusion, CMTF, and the standard Transformer remain below the leading multimodal methods in these settings. The ranking change makes the comparison across scenarios informative: a method that performs best among the baselines under one mixture need not retain that position under another. However, our proposed RTR retains the highest reported mean across all four settings.

\textbf{Simulations and Propositions.} Propositions 1 and 2 distinguish factor representations from the bridge states assigned to their connections, while Proposition 7 describes how factor features and decoded bridges enter the final representation. The four scenarios give task-level context to these design choices by evaluating the full model under different within- and cross-modality signal mixtures. The reported performance is consistent with the usefulness of the combined representation, but the comparisons do not isolate the contribution of a particular bridge or establish that its entries recover the generating relationships. The fixed modality count and tensor orders place all four scenarios in the same structural-size regime of Propositions 3, 4, and 8. This study therefore examines prediction across signal mixtures, rather than scaling with the number of bridges, state size, or input shape. Likewise, the initialization and residual-update properties in Propositions 5 and 6 are not tested directly by classification metrics. Their algebraic statements and the predictive comparisons address different questions.

The results favor the full RTR pipeline across the four prescribed settings, while also showing that higher generation-parameter values do not uniformly yield higher accuracy. The relevant practical lesson is to assess a fusion method across contrasting signal mixtures rather than infer its utility from a single setting. The reported means and standard deviations describe performance and variation across the four repeats; they are not a test of statistical significance. The evidence supports a comparison of complete predictive pipelines within these scenarios, rather than a causal attribution of the performance differences to individual architectural components.

\section{Case Study}
\label{sec:real_world_experiments}
This section evaluates the proposed RTR framework on two real-world multimodal learning tasks: emotion recognition using the SEED-IV dataset \citep{zheng2019emotionmeter} and brain-activity classification using a simultaneously collected EEG-fMRI dataset \citep{zhao2023robust}. The two cases represent different data regimes in terms of sample size and modality structure. SEED-IV provides a relatively large sample set comprising two physiological modalities with related temporal structures, whereas the EEG-fMRI dataset contains fewer samples and combines electrophysiological signals with functional imaging measurements that exhibit different structural characteristics. Therefore, evaluating both datasets allows us to examine whether RTR can generalize across different data scales and degrees of modality heterogeneity.

\subsection{Experimental Setup}
\label{sec:case_experimental_setup}

\subsubsection{Datasets and baselines.}
The first case considers four-class emotion recognition using the SEED-IV dataset, which contains EEG and eye-movement modalities and provides a relatively large-sample setting. We use the officially preprocessed smooth EEG and smooth eye-movement data. Each trial is divided into non-overlapping temporal sequences using a sequence length and stride of 5, such that each constructed sample consists of five consecutive time windows without overlap. The resulting EEG input has dimensions \(62\times 5\times 20\), corresponding to 62 EEG channels, sequence length of 5, and 20 concatenated EEG features. The eye-movement input has dimensions \(31\times 5\), representing 31 eye-movement features over the same five temporal windows. The second case is binary brain-activity classification using a simultaneously acquired EEG-fMRI dataset. Each trial contains an EEG tensor \(\mathcal X_n\in\mathbb R^{16\times 34\times 121}\), an fMRI matrix \(\mathbf Y_n\in\mathbb R^{16\times 197}\), and a binary label \(L_n\in\{1,2\}\) indicating standard or target stimulus. The dataset contains 120 trials in total, with 60 trials from each class. Compared with SEED-IV, this dataset provides a small-sample and more heterogeneous multimodal setting. For both cases, the samples are randomly divided into $70\%$ training, $10\%$ validation, and $20\%$ testing sets. The validation set is used for model selection, and predictive performance is evaluated on the held-out test set. We compare RTR with the same set of baseline methods used in the simulation study, including SVM, MLP, Tensor Fusion, CMTF, Transformer, Multimodal Transformer, and Coupled Mamba. The same data partitions and evaluation protocol are used for all competing methods to ensure a consistent comparison.

\subsubsection{Implementation and evaluation.}
For RTR, the token dimension, sequence length, and attention dimension are set to 16, 8, and 32, respectively, with 4 attention heads and 2 transformer blocks. The model is trained using AdamW with a learning rate of \(5\times10^{-4}\), weight decay of \(5\times10^{-4}\), and dropout rate of 0.1. The batch size is set to 128 for SEED-IV and 8 for EEG-fMRI, reflecting the substantially smaller sample size and high-dimensional multimodal inputs of the EEG-fMRI dataset. Model performance is evaluated using accuracy, precision, recall, and macro-F1. Accuracy measures overall predictive correctness, while precision and recall measure reliability and sensitivity, respectively. Macro-F1 averages class-specific F1 scores and provides a class-balanced assessment of predictive performance. Together, these metrics provide a comprehensive assessment of both overall classification accuracy and performance consistency across classes.

\subsection{Results Analysis}
\label{sec:case_result_analysis}
Table~\ref{tab:main_results} reports the predictive performance of all competing methods on SEED-IV and EEG-fMRI. RTR achieves the highest mean performance across all evaluation metrics on both datasets. Specifically, on SEED-IV, RTR achieves \(98.78\%\) accuracy and \(98.61\%\) macro-F1. The strongest competing baseline is Coupled Mamba, which achieves \(98.31\%\) accuracy and \(97.92\%\) macro-F1. This corresponds to improvements of 0.47 percentage points in accuracy and 0.69 percentage points in macro-F1. Multimodal Transformer also performs competitively, with a macro-F1 of \(97.24\%\), which is consistent with the value of explicitly modeling interactions between modalities. Since these competing models already exchange information across modalities, the distinction of RTR lies in how it represents their connections. RTR maintains an explicit bridge state for each prescribed dimension or modality connection, separate from the factors it links (Propositions 1 and 2). This preserves a distinction between the representation of a tensor dimension and the representation of its relationship to another dimension or modality, making these structural roles separately examinable.

The advantage of RTR is more pronounced on EEG-fMRI. It achieves \(94.44\%\) accuracy and \(94.42\%\) macro-F1, whereas CMTF, the strongest baseline on this dataset, achieves \(92.71\%\) accuracy and \(92.68\%\) macro-F1. The corresponding improvements are \(1.73\) and \(1.74\)  percentage points, respectively. RTR also achieves higher precision and recall, indicating that its performance gain is consistent across different evaluation criteria. Notably, CMTF performs strongly in this small-sample setting, suggesting that tensor-structured representations provide a useful inductive bias when training data are limited. In comparison, the standard Transformer exhibits lower mean performance and greater variability, while Multimodal Transformer and Coupled Mamba provide improvements but remain below CMTF and RTR. The change in the leading baseline, from Coupled Mamba on SEED-IV to CMTF here, gives context to RTR's combination of tensor structure and relational learning. RTR retains dimension-specific factors and updates their connection states separately, jointly preserving dimension-level structure and modeling cross-modal relationships. Across both datasets, the proposed RTR achieves the highest reported mean on all four classification metrics, supporting the effectiveness of its combined tensor-structured and relational representation.
\begin{table*}[htbp]
\centering
\caption{Classification performance on SEED-IV and EEG-fMRI. Results are reported as mean \(\pm\) standard deviation. Bold values indicate the best performance for each metric.}
\label{tab:main_results}
\scriptsize
\setlength{\tabcolsep}{3pt}
\resizebox{\textwidth}{!}{
\begin{tabular}{lcccccccc}
\toprule
\multirow{2}{*}{Method}
&
\multicolumn{4}{c}{SEED-IV}
&
\multicolumn{4}{c}{EEG-fMRI}
\\
\cmidrule(lr){2-5}
\cmidrule(lr){6-9}
&
Acc. & Precision & Recall & Macro-F1
&
Acc. & Precision & Recall & Macro-F1
\\
\midrule
SVM
& \(84.37 \pm 0.79\) & \(84.95 \pm 0.73\) & \(84.18 \pm 0.72\) & \(84.36 \pm 0.73\)
& \(85.42 \pm 5.38\) & \(87.57 \pm 3.91\) & \(85.42 \pm 5.38\) & \(85.13 \pm 5.73\)
\\
MLP
& \(95.35 \pm 1.73\) & \(95.55 \pm 1.40\) & \(95.18 \pm 1.88\) & \(95.30 \pm 1.74\)
& \(88.54 \pm 2.08\) & \(88.74 \pm 1.95\) & \(88.54 \pm 2.08\) & \(88.53 \pm 2.09\)
\\
Tensor Fusion
& \(95.50 \pm 0.29\) & \(95.50 \pm 0.27\) & \(95.45 \pm 0.34\) & \(95.48 \pm 0.31\)
& \(88.54 \pm 7.12\) & \(90.03 \pm 6.46\) & \(88.54 \pm 7.12\) & \(88.39 \pm 7.26\)
\\
CMTF
& \(96.49 \pm 1.15\) & \(96.47 \pm 1.17\) & \(96.39 \pm 1.16\) & \(96.42 \pm 1.18\)
& \(92.71 \pm 2.08\) & \(93.38 \pm 1.93\) & \(92.71 \pm 2.08\) & \(92.68 \pm 2.10\)
\\
Transformer
& \(96.27 \pm 1.04\) & \(96.22 \pm 1.07\) & \(96.23 \pm 0.96\) & \(96.20 \pm 1.03\)
& \(85.42 \pm 7.98\) & \(86.01 \pm 8.61\) & \(85.42 \pm 7.98\) & \(85.39 \pm 7.95\)
\\
Multimodal Transformer
& \(97.25 \pm 0.63\) & \(97.27 \pm 0.60\) & \(97.21 \pm 0.67\) & \(97.24 \pm 0.64\)
& \(91.67 \pm 4.81\) & \(91.96 \pm 4.84\) & \(91.67 \pm 4.81\) & \(91.65 \pm 4.82\)
\\
Coupled Mamba
& \(98.31 \pm 0.52\) & \(98.78 \pm 0.57\) & \(98.11 \pm 0.46\) & \(97.92 \pm 0.52\)
& \(91.67 \pm 3.40\) & \(92.97 \pm 2.52\) & \(91.67 \pm 3.40\) & \(91.59 \pm 3.48\)
\\
\midrule
RTR
& \(\mathbf{98.78 \pm 0.12}\) & \(\mathbf{98.86 \pm 0.16}\) & \(\mathbf{98.78 \pm 0.14}\) & \(\mathbf{98.61 \pm 0.10}\)
& \(\mathbf{94.44 \pm 1.96}\) & \(\mathbf{95.05 \pm 1.55}\) & \(\mathbf{94.44 \pm 1.96}\) & \(\mathbf{94.42 \pm 1.99}\)
\\
\bottomrule
\end{tabular}
}
\end{table*}

\subsection{Ablation Study}
\label{sec:case_ablation}
To examine the contribution of the major components of RTR, we conduct ablation studies on both datasets by removing the dimension bridge matrices, modality bridge matrices, and feature fusion module, respectively. These components are designed to capture within-modality cross-dimensional relationships, model cross-modality dependencies, and integrate factor-level and relationship-level information into the final task representation. Therefore, the ablation study evaluates whether the performance of RTR is attributable to a particular component or to their joint contribution.

Table~\ref{tab:ablation} reports the ablation results. Removing any of the three components consistently degrades performance across both datasets and all evaluation metrics, indicating that each component contributes to the effectiveness of RTR. For example, on SEED-IV, removing the dimension bridge matrices reduces accuracy from \(98.78\%\) to \(97.21\%\), while removing the modality bridge matrices further decreases it to \(96.66\%\). These results indicate that both within-modality tensor relationships and cross-modality interactions contribute to emotion recognition, with the larger degradation from removing modality bridges suggesting a particularly important role for integrating complementary EEG and eye-movement information. The effects are more pronounced on the EEG-fMRI dataset. Removing the dimension bridge matrices reduces accuracy from \(94.44\%\) to \(89.58\%\), whereas removing the modality bridge matrices decreases it to \(83.33\%\). The substantially larger impact of modality-bridge removal highlights the importance of explicitly modeling relationships between structurally heterogeneous modalities, particularly in the small-sample EEG-fMRI setting. Together with the SEED-IV results, this pattern suggests that cross-modality relational learning becomes especially valuable when the modalities differ substantially in their structural characteristics. Removing the fusion module produces the largest performance degradation on both datasets, reducing accuracy to \(93.43\%\) on SEED-IV and \(72.92\%\) on EEG-fMRI. This finding supports the value of integrating factor-level and relationship-level information into a task-relevant representation. The particularly large degradation and increased variability on EEG-fMRI further suggest that structured fusion is important for stable prediction when training data are limited and modality heterogeneity is high.
\begin{table*}[htbp]
\centering
\caption{Ablation study of RTR on SEED-IV and EEG-fMRI. Results are reported as mean \(\pm\) standard deviation. Bold values indicate the best performance for each metric.}
\label{tab:ablation}
\scriptsize
\setlength{\tabcolsep}{3.5pt}
\resizebox{\textwidth}{!}{
\begin{tabular}{lcccccccc}
\toprule
\multirow{2}{*}{Method}
&
\multicolumn{4}{c}{SEED-IV}
&
\multicolumn{4}{c}{EEG-fMRI}
\\
\cmidrule(lr){2-5}
\cmidrule(lr){6-9}
&
Acc. & Precision & Recall & Macro-F1
&
Acc. & Precision & Recall & Macro-F1
\\
\midrule
Full RTR
& \(\mathbf{98.78 \pm 0.12}\) & \(\mathbf{98.86 \pm 0.16}\) & \(\mathbf{98.78 \pm 0.14}\) & \(\mathbf{98.61 \pm 0.10}\)
& \(\mathbf{94.44 \pm 1.96}\) & \(\mathbf{95.05 \pm 1.55}\) & \(\mathbf{94.44 \pm 1.96}\) & \(\mathbf{94.42 \pm 1.99}\)
\\
w/o Dimension Bridge Matrix
& \(97.21 \pm 0.54\) & \(97.21 \pm 0.60\) & \(97.20 \pm 0.48\) & \(97.20 \pm 0.54\)
& \(89.58 \pm 4.66\) & \(89.73 \pm 4.74\) & \(89.58 \pm 4.66\) & \(89.58 \pm 4.66\)
\\
w/o Modality Bridge Matrix
& \(96.66 \pm 0.15\) & \(96.65 \pm 0.13\) & \(96.61 \pm 0.16\) & \(96.63 \pm 0.14\)
& \(83.33 \pm 5.89\) & \(83.43 \pm 5.90\) & \(83.33 \pm 5.89\) & \(83.32 \pm 5.89\)
\\
w/o Feature Fusion Module
& \(93.43 \pm 0.88\) & \(93.50 \pm 0.82\) & \(93.28 \pm 0.84\) & \(93.36 \pm 0.84\)
& \(72.92 \pm 7.12\) & \(72.99 \pm 7.12\) & \(72.92 \pm 7.12\) & \(72.89 \pm 7.13\)
\\
\bottomrule
\end{tabular}
}
\end{table*}

Taken together, these ablations help interpret how the downstream tasks use the relational design. Removing modality bridges incurs larger accuracy losses than removing dimension bridges on both datasets: 2.12 versus 1.57 percentage points on SEED-IV, and 11.11 versus 4.86 on EEG-fMRI. Within the evaluated RTR configurations, EEG-fMRI is therefore more sensitive to removal of the explicit cross-modal path, suggesting that this path warrants particular attention when simplifying that model. These task-specific comparisons highlight the value of explicitly organizing the joint use of modalities through the relational design. Moreover, the largest reported degradation on each dataset follows removal of feature fusion. This connects the value of maintaining separate relationship states to how they enter prediction: decoded bridges both participate in structured fusion and remain separate blocks in the final representation (Proposition 7). The cases thus motivate jointly preserving factor-level structure, relational states, and their integration into the prediction stage. These comparisons provide concrete architectural guidance for combining tensor factors, relational states, and structured fusion in multimodal prediction.

\section{Conclusions}
\label{sec:conclusions}
This paper introduced RTR, a relational tensor representation learning framework for multimodal health data. Unlike conventional approaches that primarily focus on representations of individual dimensions or modalities, RTR jointly learns both factor-level and relationship-level representations. Specifically, dimension-specific factors preserve the structural information associated with individual tensor dimensions, while learnable relational bridges characterize dependencies across dimensions within each modality and across heterogeneous modalities. These representations are subsequently integrated through structured dimension-modality fusion to form the final task representation. The central contribution of RTR is a relational perspective on multimodal tensor learning, in which relationships are treated as persistent learnable representations rather than solely as intermediate computational weights. Experimental results show its effectiveness across multimodal learning problems with different data scales and degrees of modality heterogeneity. The ablation results further validate that dimension-level relational learning, cross-modality relational learning, and structured fusion each contribute to the overall performance of the RTR framework. Future research could extend RTR to streaming and incomplete multimodal tensor data, investigate adaptive bridge topologies beyond the current sparse chain structure, and study relational learning in settings involving a larger number of heterogeneous modalities.

\bibliographystyle{plainnat}
\bibliography{refer}
\end{document}